\documentclass{article}

    \usepackage[preprint]{neurips_2021}

\usepackage[utf8]{inputenc} 
\usepackage[T1]{fontenc}    
\usepackage{hyperref}       
\usepackage{url}            
\usepackage{booktabs}       
\usepackage{amsfonts}       
\usepackage{nicefrac}       
\usepackage{microtype}      
\usepackage{xcolor}         

\newcommand{\mc}[1]{\mathcal{#1}}

\usepackage{graphicx}
\usepackage{mathtools}
\usepackage{footnote}
\usepackage{float}
\usepackage{xspace}
\usepackage{multirow}
\usepackage{wrapfig}
\usepackage{framed}

\usepackage{footnote}

\makesavenoteenv{tabular}
\makesavenoteenv{table}

\newcommand{\hDelta}{{\hat{\Delta}}}
\newcommand{\Dis}{{\text{Dis}}}

\usepackage{amsmath}
\usepackage{amsfonts}
\usepackage{amssymb}
\usepackage{amsthm}
\usepackage{bm}
\usepackage{bbm}
\usepackage{mathtools}
\usepackage{enumitem}
\usepackage{thmtools,thm-restate}
\usepackage{algorithm}
\usepackage{algorithmic}
\usepackage{graphicx}
\usepackage{comment}

\theoremstyle{plain}
\newtheorem{lemma}{Lemma}[section]
\newtheorem{theorem}{Theorem}[section]
\newtheorem{proposition}{Proposition}[section]
\newtheorem{corollary}{Corollary}[section]

\theoremstyle{remark}
\newtheorem{remark}{Remark}[section]

\def\Ebb{\mathbb{E}}

\def\*{\star}

\DeclareMathOperator*{\argmin}{arg\,min}
\DeclareMathOperator*{\argmax}{arg\,max}

\newcommand{\E}{\Ebb}
\newcommand{\Var}{\mathrm{Var}}

\usepackage{graphicx}
\usepackage{mathtools}
\usepackage{footnote}
\usepackage{float}
\usepackage{xspace}
\usepackage{multirow}
\usepackage{wrapfig}
\usepackage{framed}
\usepackage{xcolor}

\newcommand{\field}[1]{\mathbb{#1}}

\newcommand{\fP}{\field{P}}

\newcommand{\fN}{\field{N}}

\newcommand{\calI}{{\mathcal{I}}}

\newcommand{\calD}{{\mathcal{D}}}
\newcommand{\calE}{{\mathcal{E}}}
\newcommand{\calH}{{\mathcal{H}}}

\newcommand{\calX}{{\mathcal{X}}}
\newcommand{\calZ}{{\mathcal{Z}}}

\newcommand{\order}{\ensuremath{\mathcal{O}}}
\newcommand{\otil}{\ensuremath{\widetilde{\mathcal{O}}}}

\newcommand{\inprod}[1]{ \left\langle {#1} \right\rangle }

\newcommand{\one}{\boldsymbol{1}}
\newcommand{\hout}{h_{\textrm{out}}}
\newcommand{\Ctotal}{C_{\textrm{total}}}
\newcommand{\bCtotal}{\overline{C}_{\textrm{total}}}
\newcommand{\avgR}{\overline{R}}
\newcommand{\Bernouli}{\text{Ber}}
\title{Corruption Robust Active Learning}

\author{%
    Yifang Chen, Simon S. Du,  Kevin Jamieson\\
    Paul G. Allen School of Computer Science \& Engineering\\
    University of Washington, Seattle,WA\\
    \texttt{ \{yifangc, ssdu, jamieson \}@cs.washington.edu}\\
}

\begin{document}

\maketitle
\begin{abstract}
We conduct theoretical studies on streaming-based active learning for binary classification under unknown adversarial label corruptions. 
In this setting, every time before the learner observes a sample, the adversary decides whether to corrupt the label or
not.
First, we show that, in a benign corruption setting (which includes the misspecification setting as a special case),
with a slight enlargement on the hypothesis elimination threshold,
the classical RobustCAL framework can
(surprisingly) achieve nearly the same label complexity guarantee as in the non-corrupted setting. However, this algorithm can fail in the general corruption setting.
To resolve this drawback, we propose a new algorithm which is provably correct without any assumptions on the presence of corruptions.
Furthermore, this algorithm enjoys the minimax label complexity in the non-corrupted setting (which is achieved by RobustCAL) and only requires $\tilde{\mathcal{O}}(C_{\mathrm{total}})$ additional labels in the corrupted setting to achieve $\mathcal{O}(\varepsilon + \frac{C_{\mathrm{total}}}{n})$, where $\varepsilon$ is the target accuracy, $C_{\mathrm{total}}$ is the total number of corruptions and $n$ is the total number of unlabeled samples.



\end{abstract}
\section{Introduction}
An active learning algorithm for binary classification aims to obtain the best hypothesis (classifier) from some given hypothesis set while requesting as few labels as possible. Under some favorable conditions,  active learning algorithms can require exponentially fewer labels then passive, random sampling \citep{hanneke_theory_2014}. 
Active learning is ideally suited for applications where large datasets are required for accurate inference, but the cost of paying human annotators to label a dataset is prohibitively large \citep{joshi,yang2015multi,beluch2018power}.
A bit more formally, for an example space $\calX$ (such as a set of images) and label space $\{0,1\}$ (like whether the image contains a human-made object or not), let $\calH$ be a hypothesis class such that for each $h \in \calH$, we have $h: \calX \to \{0,1\}$. After a certain number of labels are requested, the learner will output a target hypothesis $\hout \in \calH$. In this paper, we consider the streaming setting where at each time $t$ nature reveals $x_t \sim \calD_X$ and an active learning algorithm must make the real-time decision on whether to request the corresponding label $y_t$ or not.
Such a streaming setting of active learning is frequently encountered in online environments such as learning a spam filter or fraud detection (i.e., mark as spam/fraudulent and do not request the label, or send to inbox/expert to obtain a label).

This paper is interested in a setting when the requested label $y_t$ is potentially \emph{corrupted} by an adversary. 
That is, when requesting the label for some example $x_t \in \mc{X}$, if uncorrupted the learner will receive a label drawn according to the ``true'' conditional label distribution, but if corrupted, the learner will receive a label drawn from an \emph{arbitrary} distribution decided by an adversary.  
This setting is challenging because the learner has no a priori knowledge of when or how many corruptions will occur. And if the learner is collecting data adaptively, he may easily be misled into becoming confident in an incorrect belief, collect data based on that belief, and never recover to output an accurate classifier even if the adversary eventually stops serving corrupted labels later. 
This greatly contrasts with the passive setting (when all labels are observed) where as long as the number of corrupts grows sub-linearly over time, the effect of the corruptions will fade and the empirical risk minimizer will converge to an accurate classifier with respect to the uncorrupted labels.

The source of corruptions can come from  automatic labeling, non-expert labeling, and, mostly severely, adaptive data poisoning adversaries.  Particularly, with the rise of crowdsourcing, it is increasingly feasible for such malicious labelers to enter the system (\cite{miller2014adversarial}). 
There have been many prior works that consider robust \emph{offline} training using corrupted labels (i.e., the passive setting) \citep{hendrycks2018using,pmlr-v97-yu19b}.
Correspondingly, related corruption settings have also been considered in online learning \citep{gupta2019better,zimmert2019optimal,wei2020taking} and reinforcement learning \citep{lykouris2020corruption,chen2021improved}. However, there is a striking lack of such literature in the active learning area. 
Existing disagreement-based active learning algorithms nearly achieve the minimax label complexity for a given target accuracy when labels are trusted \citep{hanneke_theory_2014}, but they fail to deal with the case where the labels are potentially corrupted. 


\paragraph{Our contributions: }
In this paper, we study active learning in the agnostic, streaming setting where an unknown number of labels are potentially corrupted by an adversary. 
We begin with the performance of existing baseline algorithms.
\begin{itemize}
    \item Firstly, we analyze the performance of empirical risk minimization (ERM) for the passive setting where all labels are observed, which will output an $\left(\varepsilon + \frac{R^*C_\textrm{total}}{n}\right)$-optimal hypothesis as long as $n \gtrapprox \frac{1}{\epsilon} + \frac{R^*}{\varepsilon^2}$, where $R^*$ is the risk of best hypothesis. This result serves as a benchmark for the following active learning results (Section~\ref{sec: passive learning}).
\end{itemize}
If we assume that the disagreement coefficient, a quantity that characterizes the sample complexity of active learning algorithms~\citep{hanneke_theory_2014}, is some constant, then we obtain the following results for active learning. 
\begin{itemize}
    \item Secondly, we analyze the performance of a standard active learning algorithm called RobustCAL \citep{balcan2009agnostic,dasgupta2007general,hanneke_theory_2014} under a benign assumption on the corruptions (misspecification model is a special case under that assumption).We show that, by slightly enlarging the hypothesis elimintion threshold, this algorithm can achieve almost the same label complexity as in the non-corrupted setting.  That is, the algorithm will output an $\order(\varepsilon + \frac{R^*\Ctotal}{n})$-optimal hypothesis as long as $n \gtrapprox \frac{R^*}{\varepsilon^2} + \frac{1}{\varepsilon} $ with at most $\otil\left(R^*n + \log(n)\right)$ number of labels (Section~\ref{sec: RobustCal}).
    \item Finally and most importantly, in the general corruption case without any assumptions on how corruptions allocated, we propose a new algorithm that matches RobustCAL in the non-corrupted case and only requires $\otil(C_\textrm{total})$ additional labels in the corrupted setting. That is, the algorithm will output an $\left(\varepsilon+ \frac{C_\textrm{total}}{n}\right)$-optimal hypothesis as long as $n \gtrapprox \frac{1}{\varepsilon^2}$  with at most $\otil\left((R^*)^2n + \log(n) + \Ctotal\right)$ number of labels. Besides, this algorithm also enjoys an improved bound under a benign assumption on the corruptions. That is, the algorithm will output an $\left(\varepsilon+ \frac{R^*C_\textrm{total}}{n}\right)$-optimal hypothesis with at most $\otil\left((R^*)^2n + \log(n) + R^*\Ctotal\right)$ number of labels  (Section~\ref{sec: main algo})
\end{itemize} 
Note that $\Ctotal$ can be regarded as a fixed budget or can be increasing with incoming samples. In the latter case, $\Ctotal$ in the second and third result can be different since they may require different $n$. Detailed comparison between these two results will be discussed in corresponding sections.

\paragraph{Related work: } 
For nearly as long as researchers have studied how well classifiers  generalize beyond their performance on a finite labelled dataset, they have also been trying to understand how to minimize the potentially expensive labeling burden.  
Consequently, the field of active learning that aims to learn a classifier using as few annotated labels as possible by selecting examples sequentially is also somewhat mature  \citep{settles2011theories,hanneke_theory_2014}.
Here, we focus on just the agnostic, streaming setting where there is no relationship assumed a priori between the hypothesis class and the example-label pairs provided by Nature. 
More than a decade ago, a landmark algorithm we call RobustCAL was developed for the agnostic, streaming setting and analyzed by a number of authors that obtains nearly minimax performance \citep{balcan2009agnostic,dasgupta2007general,hanneke_theory_2014}. 
The performance of RobustCAL is characterized by a quantity known as the disagreement coefficient that can be large, but in many favorable situations can be bounded by a constant $\theta^*$ which we assume is the case here.  
In particular, for any $\epsilon >0$, once Nature has offered RobustCAL $n$ unlabelled samples, RobustCAL promises to return a classifier with error at most $\sqrt{ R^* \log(|\mc{H}|)/ n} + \log(|\mc{H}|) /n$ and requests just $n R^*  + \theta^* ( \sqrt{n R^* \log(|\mc{H}|)} + \log(|\mc{H}|) )$ labels with high probability. 
Said another way, RobustCAL returns an $\epsilon$-good classifier after requesting just $\theta^* ( (R^*)^2/ \epsilon^2 + \log(1/\epsilon) )$ labels.  
If $\theta$ is treated as an absolute constant, this label complexity is minimax opitmal \citep{hanneke_theory_2014}.
While there exist algorithms with other favorable properties and superior performance under special distributional assumptions (c.f., \cite{zhang2014beyond,koltchinskii2010rademacher,balcan2007margin,balcan2013active,huang2015efficient}), we use RobustCAL as our benchmark in the uncorrupted setting.
We note that since RobustCAL is computationally inefficient for many classifier classes of interest, a number of works have addressed the issue at the cost of a higher label sample complexity \citep{beygelzimer2009importance,beygelzimer2010agnostic,hsu2010algorithms,krishnamurthy2017active} or higher unlabeled sample complexity \citep{huang2015efficient}. 
Our own work, like RobustCAL, is also not computationally efficient but could benefit from the ideas in these works as well. 

To the best of our knowledge, there are few works that address the non-IID active learning setting, such as the corrupted setting of this paper.
Nevertheless, \cite{miller2014adversarial} describes the need for robust active learning algorithms and the many potential attack models.
While some applied works have proposed heuristics for active learning algorithms that are robust to an adversary \citep{deng2018adversarial, pi2016defending}, we are not aware of any that are provably robust in the sense defined in this paper. 
Active learning in crowd-sourcing settings where labels are provided by a pool of a varying quality of annotators, some active learning algorithms have attempted to avoid and down-weight poorly performing annotators, but these models are more stochastic than adversarial \citep{khetan2016achieving}. 
The problem of selective sampling or online domain adaptation studies the setting where $P(Y_t = 1 | X_t = x)$ remains fixed, but $P(X_t=x)$ drifts and the active learner aims to compete with the best online predictor that observes all labels \citep{yang2011active,dekel2012selective,hanneke2021toward,chen2021active}.
Another relevant line of work considers the case where the distribution of the examples drifts over time (i.e., $P(X_t = x)$) \citep{rai2010domain} or the label proportions have changed (i.e, $P(Y_t = 1)$) \citep{zhao2021active}, but the learner is aware of the time when the change has occurred and needs to adapt. 
These setting are incomparable to our own.

Despite the limited literature in active learning, there have been many existing corruption-related works in the related problem areas of multi-arm bandits (MAB), linear bandits and episodic reinforcement learning.  To be specific, for MAB, \citet{gupta2019better} achieves $\tilde{\order}(\sum_{a\neq a^*} \frac{1}{\Delta_a}+KC)$ by adopting a sampling strategy based on the estimated gap instead of eliminating arms permanently. Our proposed algorithm is inspired by this ``soft elimination" technique and requests labels based on the estimated gap of each hypothesis.  Later, \citet{zimmert2019optimal} achieves a near-optimal result $\tilde{\order}\left(\sum_{a\neq a^*} \frac{1}{\Delta_a}+\sqrt{\sum_{a\neq a^*} \frac{C}{\Delta_a}}\right)$ in MAB by using Follow-the-Regularized Leader (FTRL) with Tsallis Entropy. How to apply the FTRL technique in active learning, however, remains an open problem. 
Besides MAB, \cite{lee2021achieving} achieves $\otil\left( \text{GapComplexity} + C\right)$ in stochastic linear bandits. 
We note that the linear bandits papers of  \cite{lee2021achieving} and \cite{camilleri2021highdimensional} both leverage the Catoni estimator that we have repurposed for robust gap estimation in our algorithm. Finally, in the episodic reinforcement learning, \cite{lykouris2020corruption} achieves $\otil\left( C \cdot\text{GapComplexity} + C^2\right)$ in non-tabular RL and \cite{chen2021improved} achieves $\otil\left( \text{PolicyGapComplexity} + C^2\right)$ in tabular RL.



\section{Preliminaries}
\label{sec: preliminary}

\paragraph{General protocol: } A hypothesis class $\mc{H}$ is given to the learner such that for each $h \in \mc{H}$ we have $h : \mc{X} \rightarrow \{0,1\}$. 
Before the start of the game, Nature will draw $n$ unlabeled samples in total. 
At each time $t \in \{1,\dots,n\}$, nature draws $(x_t,y_t) \in \calX \times \{0,1\}$ independently from a joint distribution $\calD_t$, the learner observes just $x_t$ and chooses whether to request $y_t$ or not. Note that in this paper, we assume $\calX$ is countable, but it can be directly extend to uncountable case. Next, We denote the expected risk of a classifier $h \in \mc{H}$ under any distribution $\calD$ as $R_\calD(h) = \E_{x,y\sim \calD}\left( \one\{ h(x) \neq y\}\right)$, the marginalized distribution of $x$ as $\nu$ and probability of $y= 1$ given $x$ and $\calD$ as $\eta^x$. Finally we define $\rho_\calD(h,h') = \E_{x \sim \nu} \one\{h(x) \neq h'(x)\}$.

\paragraph{Uncorrupted model: } 
In the traditional uncorrupted setting, there exists a fixed underlying distribution $\calD_*$ where each $(x_t,y_t)$ is drawn from this i.i.d distribution. Correspondingly, we define the marginalized distribution of $x$ as $\nu_*$ and probability of $y=1$ given $x$ and $\calD_*$ as $\eta_*^x$.

\paragraph{Oblivious adversary model: }
In the corrupted setting, the game at time $t$ is corrupted if  $(x_t,y_t)$ is drawn from some corrupted distribution $\calD_t$ that differs from the base $\calD_*$. 
At the start of the game the adversary chooses a sequence of functions $\eta_t^x : \mc{X} \rightarrow [0,1]$ for all $t \in \{1,\dots,n\}$. 
    The corruption level at time $t$ is measured as
    \begin{align*}
    c_t = \max_{x \in \calI} |\eta_*^x - \eta_t^x|,
\end{align*}
and the amount of corruptions during any time interval $\calI$  as $C_\calI = \sum_{t \in \calI} c_t$. Correspondingly, we define $\Ctotal = C_{[0,n]}$.
Then, Nature draws $x_t \sim \nu_*$ for each $t \in \{1,\dots,n\}$ so that each $x_t$ is independent of whether $y_t$ was potentially corrupted or not. 
One notable case case of the oblivious model is the $\gamma$-misspecification model. In the binary classification setting, it is equivalent to
\begin{align*}
    \eta_t^x = (1-\gamma)\eta_*^x + \gamma \Tilde{\eta}_t^x, \forall x.t.
\end{align*}
where $\Tilde{\eta}_t^x$ can be any arbitrary probability. Such label contamination model can be regarded a special case of corruption where for each $t$, 
\begin{align*}
    c_t = \max_x |\eta_t^x - \eta_*^x| = \gamma \max_x|\eta_*^x - \Tilde{\eta}^x| \leq \epsilon.
\end{align*}



\paragraph{Other notations: } For convenience, we denote $R_{\calD_t}(h)$ as $R_t(h)$, $R_{\calD_*}(h)$ as $R_*(h)$, $\rho_{\calD_t}(h,h') = \rho_t(h,h')$ and $\rho_{\calD_*}(h,h') = \rho_*(h,h')$.  We also define an average expected risk that will be used a lot in our analysis, $\bar{R}_\calI(h) = \frac{1}{|\calI|} \sum_{t \in \calI} R_t(h)$. In addition, we define $h^*= \argmin R_*(h)$, $R^* = R_*(h^*)$ and the gap of the suboptimal classifier $h$ as $\Delta_h = R_*(h) - R^*$. 


\paragraph{Disagreement coefficient: } For some hypothesis class $\calH$ and subset $V \subset \calH$, the region of disagreement is defined as
$
\Dis(V)=\left\{x \in \calX: \exists h, h' \in \calH \text { s.t. } h(x) \neq h'(x)\right\},
$
which is the set of unlabeled examples $x$ for which there are hypotheses in $V$ that disagree on how
to label $x$. Correspondingly, the disagreement coefficient of $h \in \calH$ with respect to a hypothesis class $\calH$ and. distribution $\nu_*$ is defined as
$$
\theta^*(r_0)=\sup_{r \geq r_0} \frac{\mathbb{P}_{X \sim \nu_*}(X \in \Dis(B(h^*, r)))}{r}.
$$

\section{Passive Learning in the Corrupted Setting}
\label{sec: passive learning}
We first analyze the performance of empirical risk minimization (ERM)  for passive learning in the corrupted setting as a benchmark. 

\begin{theorem}[Passive Learning]
\label{them: passive learning}
After $n$ labeled samples, if $\hout = \argmin_{h} \hat{R}_{[1,n]}(h)$ is the empirical risk minimizer, then with probability at least $1-\delta$, we have
\begin{align*}
    R_*(\hout) - R^* \leq \frac{\log(|\calH|/\delta)}{n} +\sqrt{\frac{8R^*\log(|\calH|/\delta)}{n}}+ \frac{8\Ctotal}{n}R^* +  \frac{5\log(|\calH|/\delta)}{n} \frac{1}{(1-\frac{4\Ctotal}{n})^2}, 
\end{align*}
This implies that, as long as $\Ctotal$ is small than some fraction of $n$, e.g., $\Ctotal \leq \frac{n}{8}$, we can obtain
$R_*(\hout) - R^* \leq \varepsilon + \frac{\Ctotal}{n}R^*$
whenever
\begin{align*}
    n \geq \frac{2\log(|\calH|/\delta)}{\varepsilon} + \frac{8R^*\log(|\calH|/\delta)}{\varepsilon^2}
\end{align*}
\end{theorem}
\paragraph{Proof Sketch } By using Bernstein inequality and the definition of corruptions, we can get
\begin{align*}
    &R_*(\hout) - R^* \\
    &\leq \frac{4\Ctotal}{n}\max\{R_*(\hout) - R^*, 2R^* \}
        + \sqrt{\frac{4\log(|\calH|/\delta)\max\{R_*(\hout) - R^*, 2R^* \}}{n}} + \frac{\log(|\calH|/\delta)}{n}
\end{align*}
Then we can directly get the result by solving this inequality. We postpone the details into Appendix~\ref{sec(app): passive learning}.

In addition to this result providing a benchmark, this passive learning result also inspires our analysis of RobustCAL in the corrupted setting as we will show in the next section.

\section{Robust CAL in the Corrupted Setting}
\label{sec: RobustCal}


We restate the classical RobustCAL \citep{balcan2009agnostic,dasgupta2007general,hanneke_theory_2014} in Algorithm~\ref{alg:robustcal} with slightly enlargement on the confidence threshold used in the elimination condition (Line~\ref{line: elimination condition}). 
The additional term $\frac{1}{2}\hat{\rho}_t(h,\hat{h}_t)$ ensures robustness because each $(R_t(h) - R_t(h')$ will be corrupted at most $2\rho_*(h,h')c_t$.
In the theorem below we show that, it can achieve the similar label complexity result as in the non-corrupted setting as long as the growth rate of corruptions is at most in a certain fraction of number of unlabeled samples.

\begin{algorithm}[t] 
\caption{RobustCAL (modified the elimination condition)}
\begin{algorithmic}[1] 
\label{alg:robustcal}
\STATE \textbf{Input: } confidence parameter $\delta$
\FOR{ $t=1,2,\ldots,n$ }
\STATE Nature reveals unlabeled data point $x_t$
\IF{ $x_t \in \Dis(V_t)$}
    \STATE Query $y_t$ and set $\hat{l}_t(h) = \one\{h(x_t) \neq y_t\}$ for all $h \in \calH$
\ENDIF
\IF{$\log(t) = \fN$}
    \STATE Set $\hat{L}_t(h) = \frac{1}{t}\sum_{s \in t} \hat{l}_s(h)$ and $\hat{h}_t = \argmin_{h\in V_t} \hat{L}_t(h)$ 
    \STATE Set $\hat{\rho}_t(h,h') = \frac{1}{t}\sum_{t} \one\{ h(x_t) \neq h'(x_t)\}$ and $\beta_t = \log(3\log(t)|\calH|^2/\delta)$
    \STATE Set 
    $V_{t+1} = \left\{h \in V_{\log(t)}: \hat{L}_t(h) - \hat{L}_t(\hat{h}_t)
    \leq \sqrt{\frac{2\beta_t\hat{\rho}_t(h,\hat{h}_t)}{t}} + \frac{3\beta_t}{2t} + \frac{1}{2}\hat{\rho}_t(h,\hat{h}_t)\right\}$
    \label{line: elimination condition}
\ELSE
    \STATE $V_{t+1} = V_t, \beta_{t+1} = \beta_t, \hat{h}_{t+1}$
\ENDIF
\ENDFOR
\STATE \textbf{Output: } $ \argmin_{h\in V_t} \hat{L}_t(h)$
\end{algorithmic}
\end{algorithm} 


\begin{theorem}
\label{prop: robust cal for known C}
Suppose the $C_{[0,t]}\leq \frac{t}{8}$ for all $t \in \{ \log(t) = \fN\}$, for example, the $(1/8)$-misspecification model. Then with high probability as least $1-\delta$, for any $n \geq (\frac{8R^*}{\varepsilon^2} + \frac{22}{\varepsilon})\log(\log(n)|\calH|^2/\delta)$, we have $R_{\hout} - R^* \leq \varepsilon + \order(\frac{R^*\Ctotal}{n})$ with label complexity at most
\begin{align*}
    \order\left(\theta^*(14R^* + 120\frac{\log(\log(n)|\calH|^2/\delta)}{n})\log(\log(n)|\calH|^2/\delta)\left(R^*n + \log(n)\right) \right) 
\end{align*}
\end{theorem}

\begin{remark}
In Appendix~\ref{sec(app): Why vanilla Robust CAL does not work?}, we show the necessity of enlarging the threshold in line~\ref{line: elimination condition} from the original 
\begin{align*}
    V_{t+1} = \left\{h \in V_{\log(t)}: \hat{L}_t(h) - \hat{L}_t(\hat{h}_t)
    \leq o\left( \sqrt{\frac{2\beta_t\hat{\rho}_t(h,\hat{h}_t)}{t}} + \frac{\beta_t}{t} \right)\right\}.
\end{align*}
by giving an counter-example. The counter-example shows that, when $R^* \gg 0$, the best hypothesis will be eliminated under the original condition even the ``$C_{[0,t]}\leq \frac{t}{8}$ for all $t \in \{ \log(t) = \fN\}$" assumption is satisfied.
\end{remark}

\paragraph{Proof Sketch} 
For correctness, it is easy to show by Bernstein inequality. For the sample complexity, Theorem~\ref{them: passive learning} implies that, for any interval $[0,t]$, as long as $C_{[0,t]} \leq \frac{t}{8}$, the learner can always identify hypothesis which are $\order(R^*+\frac{1}{n})$-optimal. Therefore, we get the probability of query as 
\begin{align*}
    \fP\left(x_{t+1} \in \Dis(V_{t+1}) \right)
    \leq \fP\left( \exists h \in V_{t+1}: h(x_t) \neq h^*(x_t), 
    \Delta_h \leq \order\left(R^* + \frac{\beta_t}{t}\right)\right) 
\end{align*}
Then by standard analysis we can connect this disagreement probability with the disagreement coefficient to get the final bound. One thing to note is that, at the first glance $\hat{\rho}_t(h,\hat{h}_t) $ might be much larger than the other two terms since it can goes to $1$, which possibly renders a worse label complexity. Here we give an intuitive explanation on why this threshold is fine: If $\hat{\rho}_t(h,\hat{h}_t) $ is close to the $|R(h) - R(\hat{h}_t)|$, then we can achieve the inequality above by using some self-bounding techniques. If $\hat{\rho}_t(h,\hat{h}_t) $ is close to the $R^*$, then we can directly get some $R^*$-dependent term in the target bound. The full proof is deferred to Appendix~\ref{sec(app): robust cal for known C}. 

\paragraph{Comparison between the modified RobustCal and passive learning: } Assume disagreement coefficient is a constant. In the non-corrupted case, the algorithm achieves the same performance guarantee as the vanilla Robust CAL. In the corrupted case, we still get the same accuracy as in Theorem~\ref{them: passive learning} with at most $\otil(R^*n + \log(n))$ number of labels, which is the same as the non-corrupted case. 

\paragraph{Discussion on the ``$C_{[0,t]} \leq \frac{t}{8}$ for all the $\{t| \log(t) \in \fN\}$" condition: } This condition can be reduced to the $(1/8)$-misspecification model as defined in Section~\ref{sec: preliminary} since $C_\calI \leq \frac{|\calI|}{8}$ for any $\calI$.
But this condition does not contain the case where an adaptive poisoning adversary corrupts all the labels at the earlier stage and stop corrupting later, which still ensures the small total amount of corruptions, but will clearly mislead the algorithm to delete a true best hypothesis $h^*$. 
\textit{ In the Section~\ref{sec: main algo}, we will show a more general result that applies to scenarios beyond $C_{[0,t]} \leq \frac{t}{8}$.
} 

\section{Main algorithm - CALruption}
\label{sec: main algo}


\subsection{Algorithm}
\begin{algorithm}[t] 
\caption{CALruption 
}
\begin{algorithmic}[1] 
\label{alg:main}
\STATE \textbf{Initialize: } $\beta_3= 2\log(\frac{3}{2}\lfloor\log(n) \rfloor |\calH|^2/\delta), \beta_1 = 32*640\beta_3, \beta_2 = \frac{5}{32},\epsilon_i = 2^{-i},N_l = \beta_1 \epsilon_l^{-2}$, $\hDelta_h^{0} = 0, V_1^0 = \calZ$ and $\tau_1= 1,q_l^x = 1$ for all $x \in \calX$ 
\FOR{$t=1,2,\ldots,n$}
\STATE Nature reveals unlabeled data point $x_t$
\STATE Set $Q_t \sim \text{Ber}(q_l^x)$ and request $y_t$ if $Q_t = 1$. \label{line: query with prob}
\STATE Set estimated loss for all $h \in \calH$ as 
    $\hat{\ell}_t(h) = \frac{\one\{h(x_t)\neq y_t\}}{q_l^x} Q_t$
\label{line: loss estimator}
\IF{$t = \tau_l+N_l-1 $ } 
    \STATE Set
        $
        \hat{\rho}_l(h,h') = \frac{1}{N_l} \sum_{t \in \calI_l} \one\{h(x_t) \neq h'(x_t)\} 
        $ for all $h,h' \in \calH$
    \label{line: begin calcualting $q$}
    \label{line: estimate disagreement}
    \STATE For each $(h,h')$, set $W_l^{h,h'} = \text{RobustEstimator}\left(\{ \hat{\ell}_t(h) - \hat{\ell}_t(h')\}_{t \in \calI_l}\right)$, which satisfies that,
    with probability at least $1-\delta$, 
    \begin{align}
        \label{condition: robust estimator}
        |(\hat{R}_l(h)-\hat{R}_l(h')) - W_l^{h,h'}|
        \leq \sqrt{ \frac{10\beta_3\hat{\rho}_l(h,h')}{N_l \min_{x \in \Dis(h,h')} q_l^x}},
    \end{align}
    where $\hat{R}_l(h) = \frac{1}{|\calI_l|}\sum_{t \in \calI_l} \E_{y \sim \Bernouli(\eta_t^{x_t})} \left[\one\{h(x_t) \neq y\}\right]$.
    \label{line: estimate gap 1}
    \STATE Set $\hat{\calD}_l = \argmin_{\calD} \max_{h,h' \in \calH} (R_{\calD}(h)-R_{\calD}(h') - W_l^{h,h'})\sqrt{\frac{\min_{x\in\Dis(h,h')} q_l^x}{\hat{\rho}_l(h,h')}}$
    \label{line: estimate model}
    \STATE Set $\hat{h}_*^l = \argmin_{h \in \calH} \left( R_{\hat{\calD}_l}(h) + \beta_2 \hat{\Delta}_h^{l-1} \right)$
    \label{line: estimate gap 2}
    \STATE Set $\hat{\Delta}_h^l = \max \left\{ \epsilon_l,   R_{\hat{\calD}_l}(h) -\left( R_{\hat{\calD}_l}(\hat{h}_*^l) + \beta_2 \hat{\Delta}_{\hat{h}_*^l}^{l-1} \right) \right\} $
    \label{line: estimate gap 3}
    \STATE Construct $V_{l+1}^i$ for all $i=0, 1,2,\ldots,l$, such that,  
    \begin{align*}
        &\hat{\Delta}_h^l \leq \epsilon_i, \forall h \in V_{l+1}^i \quad \text{ and }
        &\hat{\Delta}_h^l > \epsilon_i, \forall h \notin V_{l+1}^i
    \end{align*}
    Therefore, $V_{l+1}^l \subset V_{l+1}^{l-1} \subset \ldots \subset V_{l+1}^{0}$    
    \label{line: contruct hypothesis layers}
    \STATE Calculate the query probability $q_l^x$ for each $x$  as follows
\begin{align*}
    &\calZ(x) = \left\{ (h,h') \in \calH \mid x \in \Dis(\{h,h'\}) \right\}\\
    &k(h,h',l+1) = \max\{i \mid h,h'\in V_{l+1}^i\} \\
    &q_{l+1}^x = \max_{(h,h') \in \calZ(x)}\frac{\beta_1\hat{\rho}_l(h,h')}{N_{l+1}}\epsilon_{k(h,h',l+1)}^{-2}
\end{align*}    \label{line: end calcualting $q$}
    \STATE Set $\tau_{l+1} = \tau_l+N_l$ and denote the epoch $l$ as $\calI = [\tau_l,\tau_{l+1}-1]$. Set $l \leftarrow l+1$, go to the next epoch
\ENDIF
\ENDFOR
\STATE \textbf{Output: } $ h \in V_l^{l-1}$
\end{algorithmic}
\end{algorithm} 


In this section we describe our new algorithm, CALruption. The pseudo-code is listed in Algorithm~\ref{alg:main}.
Our previous analysis showed that in the agnostic setting the classical RobustCAL may permanently eliminate the best hypothesis due to the presence of corruptions. 
To fix this problem, in our CALruption algorithm, the learner never makes a  ``hard"  decision to eliminate any hypothesis. Instead, it assigns different query probability to each $x$ based on the estimated gap for each hypothesis as shown in line~\ref{line: query with prob} and \ref{line: loss estimator}, which can be regarded as ``soft elimination".
With this step, the key question becomes how to connect the estimated gaps with the query probability $q_l^x$.

We adopt the idea from the BARBAR algorithm proposed by \citet{gupta2019better} which was originally designed for multi-armed bandits (MAB).
Instead of permanently eliminating a hypothesis, the learner will continue pulling each arm with a certain probability defined by its estimated gap. 
However, the original BARBAR algorithm is mainly focused on estimating the reward of each individual arm.
This aligns with its MAB feedback structure, where only the information of the pulled arm will be gained at each time. 
In the active learning setting, we instead focus on the \emph{difference} of the risks of different hypothesis, because each time we request a label, values of all the hypothesis will be updated. Therefore, we implement a more complicated strategy to calculate the query probability at the end of each epoch $l$, as shown from line~\ref{line: begin calcualting $q$} to line~\ref{line: end calcualting $q$}. 

In line~\ref{line: estimate disagreement}, we estimate the disagreement probability for each hypothesis pair $(h,h')$ with an empirical quantity that upper bounds the expectation. In line~\ref{line: estimate gap 1}, instead of estimating the value of each hypothesis, we estimate the gap between each hypothesis pair $(h,h') $, denoted as $W_l^{h,h'},$ by any $\delta$-robust estimator that satisfies eq.~\ref{condition: robust estimator}. 
One example of $\delta$-robust estimator is Catoni estimator \citep{lugosi2019mean}. 
Note that simple empirical estimator will lead to potentially rare but large variance, which has been discussed in Stochastic rounding section in \cite{camilleri2021highdimensional}. 
But what we truly care is the gap between any hypothesis $h$ and the best hypothesis $h^*$. Therefore, inspired by \cite{camilleri2021highdimensional}, we construct such estimation by using $W_l^{h,h'}$ as shown in line~\ref{line: estimate model} to \ref{line: estimate gap 3}. Finally, we divide the hypothesis set into several layers based on the estimated gap and set the query probability for each $x$ based on the hypothesis layers, as shown in line~\ref{line: contruct hypothesis layers} and \ref{line: end calcualting $q$}. 

\begin{remark}
In Line~\ref{line: estimate model}, instead of estimating over all possible distribution $\calD$, we actually just need to estimate $\eta_*^x$ for all $x \in \{x_t\}_{t \in \calI_l}$ and set the corresponding $x$ distribution of $\calD$ as the empirical distribution of $x$ inside $\calI_l$.
\end{remark}

\begin{theorem}[CALruption]
\label{them: main}
With $n \geq 72\varepsilon^{-2}\beta_1$ number of unlabeled samples, with probability at least $1-\delta$ we can get an $h_{out}$ satisfying
\begin{align*}
    R_*(h_{out}) - R^* \leq \varepsilon + 24\frac{\bCtotal}{n}, 
\end{align*}
with label complexity as most
\begin{align*}
    \order\left(\theta^*(R^* + 3\sqrt{\frac{\beta_1}{n}} + \frac{64\bCtotal}{n})\log(\log(n) |\calH|^2/\delta)\left((R^*)^2n + \log(n)(1+\bCtotal)\right) \right) 
\end{align*}
where $\bCtotal = \sum_{l=1}^{\lfloor\log_4(n/\beta_1)\rfloor} C_{\text{epoch } l} \left(R^* \one\{\frac{C_{\text{epoch }l}}{N_l} \leq \frac{1}{32}\} + \one\{\frac{C_{\text{epoch }l}}{N_l} > \frac{1}{32}\}\right)$ and $\beta_1 = 16*640\log(\frac{3}{2}\lfloor\log(n) \rfloor |\calH|^2/\delta)$. Note that epoch l is prescheduled and not algorithm-dependent.
\end{theorem}

\begin{corollary}
\label{coro: main}
    Suppose the corruptions satisfy $\frac{C_{\text{epoch }l}}{N_l} \leq \frac{1}{32}$ for all epochs, for example, the $(1/32)$-misspecification case, then for any $n \geq 72\varepsilon^{-2}\beta_1$ number of unlabeled samples, with probability at least $1-\delta$ we can get a $h_{out}$ satisfying
\begin{align*}
    R_*(h_{out}) - R^* \leq \varepsilon + 24R^*\frac{\Ctotal}{n}, 
\end{align*}
with label complexity as most
\begin{align*}
    \order\left(\theta^*(R^* + 3\sqrt{\frac{R^*\beta_1}{n}} + \frac{64R^*\Ctotal}{n})\log(\log(n)|\calH|^2/\delta)\left((R^*)^2n + (R^*\Ctotal+1)\log(n)\right) \right) 
\end{align*}
\end{corollary}

\paragraph{Comparison with passive learning and the Calruption:} 
Consider the case where $\theta^*(\cdot)$ is a constant. The Corollary~\ref{coro: main} shows that, 
when $\frac{C_{\text{epoch }l}}{N_l} \leq \frac{1}{32}$ for all epochs, 
our algorithm achieves a similar accuracy $\order\left(\varepsilon + \frac{R^*\Ctotal}{n}\right)$ as in the passive learning case, while only requiring $\otil\left((R^*)^2 n + \log(n)(1+R^*\Ctotal) \right)$ number of labels, for $n \gtrapprox \frac{1}{\varepsilon^2}$. So if we set $n = \otil(\frac{1}{\varepsilon^2})$, then the label complexity becomes $\otil\left(\frac{(R^*)^2}{\varepsilon^2} + \log(1/\varepsilon)(1+R^*\Ctotal) \right)$, which matches the minimax label complexity in the non-corrupted case.

Going beyond the $\frac{C_{\text{epoch }l}}{N_l} \leq \frac{1}{32}$ constraint, the general Theorem~\ref{them: main} shows that, for $n \gtrapprox \frac{1}{\varepsilon^2}$, our algorithm achieves an accuracy $\order\left(\varepsilon + \frac{\Ctotal}{n}\right)$ while only requiring $\otil((R^*)^2 n + \log(n) + \Ctotal)$ number of labels no matter how corruptions are allocated. When $R^*$ is some constant, this result becomes similar to the Corollary~\ref{coro: main}. Moreover, we will argue that upper bound $\bCtotal$ by $\Ctotal$ is loose and in many case $\bCtotal$ will be close to $R^*\Ctotal $ instead of $\Ctotal$. We show one example in the paragraph below. 

\paragraph{When is Calruption better than modified Robust CAL? } Consider the case where the adversary fully corrupts some early epoch and then performs corruptions satisfying $\frac{C_{\text{epoch }l}}{N_l} \leq \frac{1}{32}$ for rest epochs. Then the modified Robust CAL will mistakenly eliminate $h^*$ so it can never achieve target result when $\varepsilon < \min_{h \in \calH} \Delta_h$ while Calruption can surely output the correct hypothesis. Moreover, according to Theorem~\ref{them: main}, since the total amount of early stage corruptions are small, so here $\bCtotal$ is close to $R^*\Ctotal$, which implies a similar result as in Corrollary~\ref{coro: main}.

\paragraph{When is Calruption worse then modified Robust CAL ?} Consider the case where the total amount of corruption is, instead of fixed, increasing with incoming unlabeled samples, for example, the misspecification case. Then $\Ctotal$ in modified Robust CAL can be $\order(\frac{R^*}{\varepsilon^2} + \frac{1}{\varepsilon})$ while $\Ctotal$ in CALruption can goes to $\order(\frac{1}{\varepsilon^2})$. Such gap comes from the extra unlabeled sample complexity, which we discuss in the paragraph below.
%

\paragraph{Discussion on the extra unlabeled samples complexity: } 
We note that we require a larger number of unlabeled data than ERM in the passive learning setting. Here we explain the reason.
Consider the version spaces $V_l^{l-1}$ for any fixed epoch $l$. In the non-corrupted setting, this version space serves the similar purpose as the active hypothesis set in Robust CAL. 
In Robust CAL, its elimination threshold is about $\otil\left(\sqrt{\frac{\rho_*(h,h')}{t}} + \frac{1}{t}\right)$ (or $\otil\left(\rho_*(h,h') + \frac{1}{t}\right)$ in our modified version) while in our CALruption, the threshold is about $\otil\left(\sqrt{\frac{1}{t}}\right)$, which is more conservative than the Robust CAL and leads to the extra unlabeled sample complexity. The reason about being conservative here is that we need more samples to weaken the effects of corruptions on our estimation. Whether such extra unlabeled samples complexity is unavoidable remains an open problem. 

\subsection{Proof sketch for Theorem~\ref{them: main}}
Here we provide main steps of the proof and postpone details in Appendix~\ref{sec(app): whole anaylisi for main algo}. 


First we show a key lemma which guarantees the closeness between $\hDelta_h^l$ and $\Delta_h$ for all $l$ and $h$.

\begin{lemma}[Upper bound and lower bound for all estimation]
\label{lem(main): upper and lower bound}
    With probability at least $1-\delta$, for all epoch $l$ and all $h \in \calH$,
    \begin{align*}
      &\hDelta_h^l \leq 2\left( \Delta_h + \epsilon_l + g_l \right),
      &\Delta_h \leq \frac{3}{2}\hDelta_h^l + \frac{3}{2}\epsilon_l + 3g_l,
    \end{align*}
    where 
    $
        g_{l} 
        = \frac{2}{\beta_1}\epsilon_l^2 \sum_{s=1}^l C_s \left(2R^*\one\left\{\frac{2C_{\calI_s}}{N_s}\leq \frac{1}{16}\right\} + \one\left\{\frac{2C_{\calI_s}}{N_s}> \frac{1}{16}\right\}\right).
    $
\end{lemma}
Here the $g_l$ term implies that, as long as the total corruption is sublinear in $n$, the misleading effects on the gap estimations will fade when the number of unlabeled samples increasing.

Based on this lemma, we can directly get another useful lemma as follows.
\begin{lemma}
\label{lem(main): upper bound the difference of classifiers}
For all epoch $l$ and layer $j$, we have
$
    \max_{h \in V_l^j} \rho_*(h,h^*)
    \leq 2R^* + 3\epsilon_j + 3g_{l-1}
$
\end{lemma}
In the following we first deal with the correctness then then sample complexity.

\textbf{Correctness.} By Lemma~\ref{lem(main): upper and lower bound}, we have
\begin{align*}
    \Delta_{h_{out}}
    \leq \frac{3}{2}\hDelta_{h_{out}}^{L-1} + \frac{3}{2}\epsilon_{L-1} + 3g_{L-1}
    \leq 6\sqrt{\frac{2\beta_1}{n}} + 24\frac{\bar{C}_{total}}{n}.
\end{align*}

\textbf{Sample complexity.}
For any $t \in \calI_l$, recall that $q_l^x = \max_{(h,h') \in \calZ(x)}\frac{\beta_1\hat{\rho}_{l-1}(h,h')}{N_l}\epsilon_{k(h,h',l)}^{-2}$, the probability of $x_t$ being queried ($Q_t = 1$) is

\begin{align*}
    \E[Q_t]
    & \leq 10\frac{\beta_1}{N_l}\sum_{x \in \calX}  P(x_t = x) \max_{h \in V_l^{j_l^x}} \rho_*(h,h^*)\epsilon_{j_l^x}^{-2} + 8\frac{\beta_1}{N_l} \\
    & \leq 10\frac{\beta_1}{N_l}\sum_{x \in \calX}  P(x_t = x) \left( 2R^*\epsilon_{j_l^x}^{-2} + 3\epsilon_{j_l^x}^{-1} + 3g_{l-1}\epsilon_{j_l^x}^{-2} \right) + 8\frac{\beta_1}{N_l} \\
    &\leq 10\frac{\beta_1}{N_l} \sum_{i=0}^{l-1} \left( 2R^*\epsilon_{i}^{-2} + 3\epsilon_{i}^{-1} + 3g_{l-1}\epsilon_{i}^{-2} \right) \fP( x \in   \Dis(V_l^i) )+ 8\frac{\beta_1}{N_l}
\end{align*}
Here $j_l^x$ is some arbitrary mapping from $\calX$ to $[l]$, which is formally defined in detailed version in Appendix~\ref{sec(app): main proof for main algo}. The first inequality comes from the closeness of estimated $\hat{\rho}_l(h,h')$ and the true  $\rho_*(h,h')$, as well as some careful relaxation. The second inequality comes from Lemma~\ref{lem(main): upper bound the difference of classifiers}. 

Now we can use the standard techniques to upper bound $\fP( x \in   \Dis(V_l^i) )$ as follows,
\begin{align*}
    \fP\left(\exists h \in V_l^i: h(x)\neq h^*(x) \right)
    &\leq \fP\left(\exists h \in \calH: h(x)\neq h^*(x), \rho_*(h,h^*) \leq 2R^* + 3\epsilon_i + 3g_{l-1} \right)\\
    &\leq \theta^*(2R^* + 3\epsilon_i + g_{l-1}) \left( 2R^* + 3\epsilon_i + 3g_{l-1}\right)
\end{align*}
where again the first inequality comes from Lemma~\ref{lem(main): upper bound the difference of classifiers}. Again we postpone the full version into Appendix~\ref{sec(app): main proof for main algo}.

Combining the above results with the fact that 
$g_l = \frac{2}{\beta_1}\epsilon_l^2\bar{C}_{l-1} $  and $\bar{C}_{l-1} \leq \sum_{s=1}^{l-1} C_{\calI_s} \leq 2\beta_1 \epsilon_{l-1}^{-2}$, we get the expected number of queries inside a complete epoch $l$ as,
\begin{align*}
    \sum_{t \in \calI_l} \E[Q_t]
    \leq 20\beta_1 \theta^*(2R^* + 3\epsilon_{l-1} + g_{l-1})
      *\left(4(R^*)^2\epsilon_l^{-2} + 12 R^*\epsilon_l^{-1} + \frac{132}{\beta_1}\bar{C}_{l-1} + 10\right)
\end{align*}

Finally, summing over all $L = \lceil \frac{1}{2}\log(n/\beta_1) \rceil$ number of epochs, for any $n$, we can get the target lable complexity.


\section{Conclusion and future works}
In this work we analyzed an existing active learning algorithm in the corruption setting, showed when it fails, and designed a new algorithm that resolve the drawback.
Relative to RobustCAL, our algorithm requires a larger number of unlabeled data. One natural question is to design a corruption robust algorithm which requires the same number of unlabeled data as RobustCAL in the non-corrupted setting.
Another potential question is that, the $\order\left(\varepsilon+\frac{\bCtotal}{n}\right)$ accuracy from Algo.~\ref{alg:main} in the general corruption case is generally worse than $\order\left(\varepsilon+\frac{R^*\Ctotal}{n}\right)$ accuracy of passive learning. Although we state that these two bounds are close in many cases, a question is if there exists an alternative algorithm or analysis that will result a more smooth final bound to interpolate between different corruptions cases.
Finally, it is also an interesting direction to design algorithms that are robust to an even stronger corruption model, e.g., adaptive adversary.

\bibliographystyle{unsrtnat}
\bibliography{ref}

\newpage
\tableofcontents
\appendix
\section{Lemmas related to corruption effects}
Here we states some basic lemmas that will be used all over the proofs.

\begin{lemma}[Corruption effects 1]
\label{lem: corruption related  1}
For any interval $\calI$ and hypothesis $h$, we have
\begin{align*}
    \frac{1}{|\calI|}\sum_{t \in \calI} \left( R_t(h) - R_*(h) \right)
    \leq \frac{C_\calI}{|\calI|}
\end{align*}
\end{lemma}
\begin{proof}
\begin{align*}
    &\frac{1}{|\calI|}\sum_{t \in \calI} \left(R_t(h) - R_*(h)\right) \\
    & =\E_{x\sim \nu_*}\frac{1}{|\calI|}\sum_{t \in \calI}\left( \E_{y \sim \eta_t^x}\left[ \one\{h(x) \neq y\}\right] - \E_{y \sim \eta_*^x}\left[ \one\{h(x) \neq y\}\right] \right) \\
    & \leq \frac{1}{|\calI|}\sum_{t \in \calI} \max_{x \in \calX}\left( \E_{y \sim \eta_t^x}\left[ \one\{h(x) \neq y\}\right] - \E_{y \sim \eta_*^x}\left[ \one\{h(x) \neq y\}\right] \right) \\
    & \leq \frac{1}{|\calI|}\sum_{t \in \calI} \max_{x \in \calX} |\eta_t^x - \eta_*^x|
    \leq \frac{C_\calI}{|\calI|}
\end{align*}
\end{proof}

\begin{lemma}[Corruption effects 2]
\label{lem: corruption related  2}
For any interval $\calI$ and hypothesis pair $h,h'$, we have
\begin{align*}
    \frac{1}{|\calI|}\sum_{t \in \calI}\left( R_t(h) - R_t(h')\right) - \left( R_*(h) - R_*(h')\right)
    \leq 2\rho_*(h,h')\frac{C_\calI}{|\calI|}
\end{align*}
\end{lemma}
\begin{proof}
\begin{align*}
    &\frac{1}{|\calI|}\sum_{t \in \calI}\left( R_t(h) - R_t(h')\right) - \left( R_*(h) - R_*(h')\right)\\
    &=\E_x \left[ \frac{1}{|\calI|}\sum_{t \in \calI}\left( \E_{t \sim \eta_t^x}\left[ \one\{h(x) \neq y\} - \one\{h'(x) \neq y\}\right] - \E_{t \sim \eta_*^x}\left[ \one\{h(x) \neq y\} - \one\{h'(x) \neq y\}\right]\right)\right] \\
    &= \E_x \left[  \frac{\one\{h(x) \neq h'(x)\}}{|\calI|}\sum_{t \in \calI}\left( \E_{t \sim \eta_t^x}\left[ \one\{h(x) \neq y\} - \one\{h'(x) \neq y\}\right] - \E_{t \sim \eta_*^x}\left[ \one\{h(x) \neq y\} - \one\{h'(x) \neq y\}\right]\right)\right]\\
    & \leq \rho_*(h,h') \left( \frac{2}{|\calI|} \sum_{t \in \calI} \max_{h \in \calH}\left(R_t(h) - R_*(h) \right) \right)\\
    &\leq 2\rho_*(h,h') \frac{C_\calI}{|\calI|}
\end{align*}

\end{proof}
\section{Analysis for Passive Learning: Proof of Theorem~\ref{them: passive learning}}
\label{sec(app): passive learning}

With probability at least $1-\delta$, we have for any $n$ samples,
\begin{align*}
    &R_*(\hout) - R^* \\
    &\leq \left( (R_*(\hout) - R^*) - (\bar{R}_{[1,n]}(\hout) -\bar{R}_{[1,n]}(h^*) )\right) + \left(\bar{R}_{[1,n]}(\hout) -\bar{R}_{[1,n]}(h^*)\right)\\
    & \leq 2\frac{\Ctotal}{n}\rho_*(\hout,h^*)
        + \left(\bar{R}_{[1,n]}(\hout) -\bar{R}_{[1,n]}(h^*)\right) \\
    & \leq 2\frac{\Ctotal}{n}\rho_*(\hout,h^*)
        + \left(\hat{R}_{[1,n]}(\hout) -\hat{R}_{[1,n]}(h^*)\right)
        + \sqrt{\rho_*(\hout,h^*)\frac{4\log(|\calH|/\delta)}{n}} + \frac{\log(|\calH|/\delta)}{n}\\
    &\leq 4\frac{\Ctotal}{n}\max\{R_*(\hout) - R^*, 2R^* \}
        + \sqrt{\max\{R_*(\hout) - R^*, 2R^* \}\frac{4\log(|\calH|/\delta)}{n}} + \frac{\log(|\calH|/\delta)}{n}
\end{align*}
where the second step can from our definition of corruptions and fact that $\nu_*$ is not corrupted (see Lemma~\ref{lem: corruption related  2} for details), third inequality comes from the Bernstein inequality and the last inequality comes from the definition of $\hout$ and the fact $\rho_*(h,h') \leq 2\max\{R_*(h)-R^*, 2R^*\}$. Now if $2R^* \geq R_*(\hout) - R^*$, then we directly get the target result. Otherwise, by solving the quadratic inequality, we have
\begin{align*}
    R_*(\hout) - R^* 
    \leq \frac{5\log(|\calH|/\delta)}{n} \frac{1}{(1-\frac{4\Ctotal}{n})^2}
\end{align*}

\section{Analysis for Robust CAL}

\subsection{Proof of Theorem~\ref{prop: robust cal for known C}}
\label{sec(app): robust cal for known C}

For convenient, for all subscripts $[0,t]$, we simply write as subscript $t$.

We first state a key lemma that is directly inspired by Theorem~\ref{them: passive learning}.
\begin{lemma}
\label{lem: passive learning}
For any $t$ that $\log(t) = \fN$, under the assumption of this theorem, as long as $h^* \in V_t$, we have
\begin{align*}
    R_*(\hat{h}_t) - R^* 
    &\leq \frac{22\log(|\calH|/\delta)}{t} +  4\frac{C_{t}}{n}R^* + \sqrt{R^* \frac{8\log(|\calH|/\delta)}{t}} \\
    &\leq  \frac{22\log(|\calH|/\delta)}{t} +  \frac{R^*}{2} + \sqrt{R^* \frac{8\log(|\calH|/\delta)}{t}} \quad \text{(By assumption on $C_t$) }\\
    &\leq \frac{26\log(|\calH|/\delta)}{t} + R^* \quad \text{(By the fact $\sqrt{AB} \leq \frac{A+B}{2}$) }
\end{align*}
\end{lemma}
\begin{proof}
With probability at least $1-\delta$, by combine the same proof steps as in  Theorem~\ref{them: passive learning} and  the fact that $\hat{R}_{[1,t]}(\hat{h}_t) -\hat{R}_{[1,t]}(h^*) = \hat{L}_t(\hat{h}_t) - \hat{L}_t(h^*) \leq 0$, we can get the similar inequality as follows
\begin{align*}
    R_*(\hat{h}_t)
    &\leq 4\frac{C_t}{n}\max\{R_*(\hat{h}_t) - R^*, 2R^* \}
        + \sqrt{\max\{R_*(\hat{h}_t) - R^*, 2R^* \}\frac{4\log(|\calH|/\delta)}{t}} + \frac{\log(|\calH|/\delta)}{t}
\end{align*}
Then again by quadratic inequality and the assumption that $\frac{C_t}{t} \leq \frac{1}{8}$, we have
\begin{align*}
    R_*(\hat{h}_t) 
    &\leq \frac{22\log(|\calH|/\delta)}{t} +  4\frac{C_{t}}{n}R^* + \sqrt{R^* \frac{8\log(|\calH|/\delta)}{t}}
\end{align*}
\end{proof}
This lemma suggests that, as long as the corruptions are not significantly large. For example, in this theorem, $C_t \leq \frac{1}{8}t$. Then the learner can still easily identify the $\otil(\frac{1}{t} + R^*)$-optimal hypothesis even in the presence of corruptions. Therefore, we can guarantee that the best hypothesis always stay in active set $V_t$ after elimination. We show the detailed as follows.

Define $\calE_1, \calE_2$ as
\begin{align*}
    &\calE_1:=
    \left\{ \forall t \text{ that } \log(t) = \fN, (\avgR_t(h) - \avgR_t(h')) - (\hat{R}_t(h)-\hat{R}_t(h')) \leq \sqrt{\frac{2\beta_t\hat{\rho}_t(h,h')}{t}}+ \frac{\beta_t}{t}\right\} \\
    &\calE_2: = 
    \left\{ \forall t \text{ that } \log(t) = \fN, (\avgR_t(h) - \avgR_t(h')) - (\hat{R}_t(h)-\hat{R}_t(h')) \leq \sqrt{\frac{2\beta_t\rho_*(h,h')}{t}}+ \frac{\beta_t}{t}\right\}\\
    &\calE_3: = 
    \left\{ \forall t \text{ that } \log(t) = \fN, |\rho_*(h,h') - \hat{\rho}_t(h,h')| \leq \sqrt{\frac{2\beta_t\hat{\rho}_t(h,h')}{t}}+ \frac{\beta_t}{t}\right\}
\end{align*}
By (empirical) Bernstein inequality plus union bound, it is easy to see $\fP(\calE_1 \cap \calE_2\cap \calE_3) \geq 1-\delta$.

\textbf{First we show the correctness.}

For any $t$ that $\log(t) = \fN$, assume that $h^* \in V_t$, then we have
\begin{align*}
    \hat{L}_t(h^*) - \hat{L}_t(\hat{h}_t)
    &=\hat{R}_t(h^*) - \hat{R}_t(\hat{h}_t) \\
    &\leq \bar{R}_t(h^*) - \bar{R}_t(\hat{h}_t) + \sqrt{\frac{\beta_t\hat{\rho}_t(h^*,\hat{h}_t)}{t}} + \frac{\beta_t}{2t} \\
    &\leq R^* - R_*(\hat{h}_t) + \sqrt{\frac{2\beta_t\hat{\rho}_t(h^*,\hat{h}_t)}{t}} + \frac{\beta_t}{t} + \rho_*(h^*,\hat{h}_t)\frac{2C_t}{t} \\
    &\leq \sqrt{\frac{2\beta_t\hat{\rho}_t(h^*,\hat{h}_t)}{t}} + \frac{\beta_t}{t} + \rho_*(h^*,\hat{h}_t)\frac{2C_t}{t} \\
    &\leq \sqrt{\frac{2\beta_t\hat{\rho}_t(h^*,\hat{h}_t)}{t}} + \frac{\beta_t}{t} 
        +  \left(\hat{\rho}_t(h^*,\hat{h}_t)+\sqrt{\frac{2\beta_t\hat{\rho}_t(h^*,\hat{h}_t)}{t}} + \frac{\beta_t}{t}\right)\frac{2C_t}{t}\\
    &\leq \sqrt{\frac{2\beta_t\hat{\rho}_t(h^*,\hat{h}_t)}{t}} + \frac{3\beta_t}{2t} + \frac{1}{2}\hat{\rho}_t(h^*,\hat{h}_t)
\end{align*}

where the first and forth inequality comes from the event $\calE_1$ and $\calE_3$, the second inequality comes from Lemma~\ref{lem: corruption related  2}, the third inequality comes from the definition of $R^*$ and last inequality comes from 
$\sqrt{\frac{2\beta_t\hat{\rho}_t(h^*,\hat{h}_t)}{t}} \leq \frac{\hat{\rho}_t(h^*,\hat{h}_t)}{2} + \frac{\beta_t}{t}$ and the assumption that $\frac{C_t}{t} \leq \frac{1}{8}$.

According to the elimination condition~\ref{line: elimination condition} in Algo.~\ref{alg:robustcal}, this implies that $h^* \in V_{t+1}$. Therefore, by induction, we get that $h^* \in V_n$. By again using Lemma~\ref{lem: passive learning}, we can guarantee that
\begin{align*}
     R_*(h_{out}) - R^* \leq  \frac{22\log(|\calH|/\delta)}{n} +  \frac{4R^*\Ctotal}{n} + \sqrt{R^* \frac{8\log(|\calH|/\delta)}{n}}
\end{align*}

\textbf{Next we show the sample complexity.}
For any $t$ that $\log(t) = \fN$ and any $h \in V_t$, we have
\begin{align*}
    \Delta_h
    & = \left(\Delta_h - (\bar{R}_t(h)-\bar{R}_t(h^*)) \right)
    +\left((\bar{R}_t(h)-\bar{R}_t(h^*)) - (\hat{R}_t(h)-\hat{R}_t(h^*))\right)
    + (\hat{R}_t(h)-\hat{R}_t(h^*))\\
    &\leq \frac{2C_t}{t}\rho_*(h,h^*)
     + \sqrt{\frac{2\beta_t\rho_*(h,h^*)}{t}} + \frac{\beta_t}{t}
     +\hat{R}_t(h)-\hat{R}_t(\hat{h}_t) \\
    &\leq \frac{1}{4}\rho_*(h,h^*)
     + \sqrt{\frac{2\beta_t\rho_*(h,h^*)}{t}} + \frac{\beta_t}{t}
     + \sqrt{\frac{2\beta_t\hat{\rho}_t(h^*,\hat{h}_t)}{t}} + \frac{3\beta_t}{2t} + \frac{1}{2}\hat{\rho}_t(h^*,\hat{h}_t)\\
    &\leq \frac{19}{24}\rho_*(h,h^*)
     + \sqrt{\frac{2\beta_t\rho_*(h,h^*)}{t}}
     +\sqrt{\frac{2\beta_t\hat{\rho}_t(h,h^*)}{t}} 
     +\sqrt{\frac{2\beta_t\hat{\rho}_t(\hat{h}_t,h^*)}{t}}
     +\frac{6\beta_t}{t} \\
     &\leq \left(\frac{19}{24}+ \frac{25}{24\beta_4} \right)\rho_*(h,h^*)
        + \frac{13}{24\beta_4}\rho_*(\hat{h}_t,h^*) + \left(2\beta_4 + 6 + \frac{21}{2\beta_4}\right)\\
    &\leq \left(\frac{19}{24}+ \frac{25}{24\beta_4} \right)\Delta_h + \frac{13}{24\beta_4}\Delta_{\hat{h}_t} + \left(2\beta_4 + 6 + \frac{21}{2\beta_4}\right)\frac{\beta_t}{t}
    + 2\left(\frac{19}{24}+ \frac{25}{24\beta_4} + \frac{13}{24\beta_4}\right)R^*\\
    &\leq \left(\frac{19}{24}+ \frac{25}{24\beta_4} \right)\Delta_h + \left(2\beta_4 + 6 + \frac{169}{12\beta_4} + \frac{21}{2\beta_4}\right)\frac{\beta_t}{t}
    + 2\left(\frac{19}{24}+ \frac{25}{24\beta_4} + \frac{13}{24\beta_4} + \frac{13}{48\beta_4}\right)R^*
\end{align*}
where the first inequality comes from the event $\calE_2$ and the definition of $\hat{h}_t$, the second inequality comes from the elimination condition~\ref{line: elimination condition} in Algo.~\ref{alg:robustcal}. For the third and forth inequality, we use the fact $\sqrt{AB} \leq \frac{A+B}{2}$ multiple times and the last inequality comes from Lemma~\ref{lem: passive learning}. 

Finally, choose $\beta_4 = 25$ and solve this inequality, we get 
$
    \Delta_h \leq \frac{120 \beta_t}{t} + 12R^*
$

Therefore, we get the probability of query as 
\begin{align*}
    \fP\left(x_{t+1} \in \Dis(V_{t+1}) \right)
    & \leq \fP\left( \exists h \in V_{t+1}: h(x_t) \neq h^*(x_t), 
    \Delta_h \leq \frac{120 \beta_t}{t} + 12R^*\right) \\
    &\leq \fP\left( \exists h \in V_{t+1}: h(x_t) \neq h^*(x_t), 
    \rho_*(h,h^*) \leq 14R^* + \frac{120 \beta_t}{t}\right)\\
    &\leq \theta^*(14R^* + \frac{120 \beta_t}{t}) \left(14R^* + \frac{120 \beta_t}{t} \right)
\end{align*}
Therefore, we get the final prove by summing this probability over all the time.

\subsection{Why vanilla Robust CAL does not work?}
\label{sec(app): Why vanilla Robust CAL does not work?}

\begin{proposition}
When $R^* \gg 0$ and the corruptions are unknown to the learner, there exists an instance and an adversary such that the vanilla Robust CAL can never output the target hypothesis.
\end{proposition}
\begin{proof}
Suppose $\calX = \{x_1,x_2,x_3\}$ where $\nu_*(x_1) = \xi_1 \gg 0, \nu_*(x_2) = \xi_2 \leq \frac{\xi_1}{64}$ and $\nu_*(x_3) = 1- \xi_1-\xi_2$. Here we further assume that $\nu$ is given to learner. For labels, we set $\eta_*^{x_1} = \frac{1}{2}, \eta_*^{x_2} = \eta_*^{x_2} = 1$. Now consider $h_1: h_1(x_1)= h_1(x_2) = h_1(x_3)=1$ and $h_2: h_2(x_1)= h_2(x_2)=0, h_2(x_3)=1$.
With some routine calculations, we can obtain that: 
\begin{align*}
    R^* = R_*(h_1) = \frac{1}{2}\xi_1,
    \quad 
    R_*(h_2) = \frac{1}{2}\xi_1 + \xi_2,
    \quad
    \rho_*(h_1,h_2) = \xi_1 + \xi_2
\end{align*}
Now suppose the adversary corrupts $\eta_*^{x_1}$ from $\frac{1}{2}$ to $\eta_s^{x_1} = \frac{15}{32}$ for all $s \leq \tau$ and will stop corrupting at certain time $\tau$. 
Consider this case $C_t \leq \frac{1}{32}t$, which satisfies our corruption assumption.

With such corruptions, we have that for any $t \leq \tau$,
\begin{align*}
    \bar{R}_t(h_1) = \frac{17}{32}\xi_1,
    \quad 
    \bar{R}_t(h_2) = \frac{15}{32}\xi_1 + \xi_2,
\end{align*}
Since $\bar{R}_t(h_2) \geq \bar{R}_t(h_1)$, so $h_2$ will never be eliminated before $\tau$. Next we show that $h_1$ can be eliminated before $\tau$.
Note that, when $\tau \geq O(\frac{1}{\xi_1})$, we can always find a proper $t \leq \tau$ such that
\begin{align*}
    \hat{R}_t(h_1) - \hat{R}_t(h_2)
    \geq \frac{1}{16}\xi_1 - \xi_2 - \otil\left(\sqrt{\frac{\xi_1+\xi_2}{t}} + \frac{1}{t}\right)
\end{align*}
In the non-corrupted setting, the confidence threshold of vanilla Robust CAL is always $\otil\left(\sqrt{\frac{\xi_1+\xi_2}{t}} + \frac{1}{t}\right)$, which can be smaller than $ \frac{1}{16}\xi_1 - \xi_2 - \otil\left(\sqrt{\frac{\xi_1+\xi_2}{t}} + \frac{1}{t}\right)$ for large enough $t$, so the above inequality shows that $h_1$ can be eliminated before $\tau$. 
This implies that, if our target accuracy $\varepsilon < \xi_2$, then the vanilla Robust CAL will never able to output the correct answer no matter how many unlabeled samples are given. On the other hand, in the passive learning, one can still output the target $h_1$ as long as $n \gg \tau$.
\end{proof}
\section{Analysis for Algo~\ref{alg:main}}
\label{sec(app): whole anaylisi for main algo}

\subsection{Notations}
Let $\calI_l$ denotes the epoch $l$, $C_l$ denotes $C_{\calI_l}$.

\subsection{Concentration guarantees on $\delta$-robust estimator}
In this section, we show the analysis by using the Catoni's estimator which is described in detail as below. Note that the same estimator has been used in previous works including \cite{wei2020taking,camilleri2021highdimensional,lee2021achieving}.

\begin{lemma}(Concentration inequality for Catoni's estimator \cite{wei2020taking})
\label{lem: concentratin bound for catoni -- original}
Let $\mathcal{F}_{0} \subset \cdots \subset \mathcal{F}_{n}$ be a filtration, and $X_{1}, \ldots, X_{n}$ be real random variables such that $X_{i}$ is $\mathcal{F}_{i}$ -measurable, $\mathbb{E}\left[X_{i} \mid \mathcal{F}_{i-1}\right]=\mu_{i}$ for some fixed $\mu_{i}$, and $\sum_{i=1}^{n} \mathbb{E}\left[\left(X_{i}-\mu_{i}\right)^{2} \mid \mathcal{F}_{i-1}\right] \leq V$ for some fixed $V .$ Denote $\mu \triangleq \frac{1}{n} \sum_{i=1}^{n} \mu_{i}$ and let $\widehat{\mu}_{n, \alpha}$ be the Catoni's robust mean
estimator of $X_{1}, \ldots, X_{n}$ with a fixed parameter $\alpha>0$, that is, $\widehat{\mu}_{n, \alpha}$ is the unique root of the function
$$
f(z)=\sum_{i=1}^{n} \psi\left(\alpha\left(X_{i}-z\right)\right)
$$
where
$$
\psi(y)=\left\{\begin{array}{ll}
\ln \left(1+y+y^{2} / 2\right), & \text { if } y \geq 0 \\
-\ln \left(1-y+y^{2} / 2\right), & \text { else }
\end{array}\right.
$$
Then for any $\delta \in(0,1)$, as long as $n$ is large enough such that $n \geq \alpha^{2}\left(V+\sum_{i=1}^{n}\left(\mu_{i}-\mu\right)^{2}\right)+2 \log (1 / \delta)$, we have with probability at least $1-2 \delta$,
\begin{align*}
\left|\widehat{\mu}_{n, \alpha}-\mu\right| 
&\leq \frac{\alpha\left(V+\sum_{i=1}^{n}\left(\mu_{i}-\mu\right)^{2}\right)}{n}+\frac{2 \log (1 / \delta)}{\alpha n}\\
&\leq \frac{\alpha\left(V+\sum_{i=1}^{n}\mu_{i}^{2}\right)}{n}+\frac{2 \log (1 / \delta)}{\alpha n}.    
\end{align*}
\end{lemma} . 

\begin{lemma}[Concentration inequality in our case]
\label{lem: concentratin bound for catoni}
For any fixed epoch $l$ and any pair of classifier $h,h' \in \calH$, as long as $N_l \geq 4\log(1/\delta)$, with probability at least $1-\delta$, we have
\begin{align*}
    |(\hat{R}_l(h)-\hat{R}_l(h')) - W_l^{h,h'}|
    \leq \sqrt{ \frac{10\log(1/\delta)\hat{\rho}_l(h,h')}{N_l \min_{x \in \Dis(h,h')} q_l^x}}
\end{align*}
where $\hat{R}_l(h) = \frac{1}{|\calI_l|}\sum_{t \in \calI} \E_{y \sim \Bernouli(\eta_t^{x_t})} \left[\one\{h(x_t) \neq y\}\right]$ (restate)
\end{lemma}
\begin{proof}
First we calculate the expectation and variance of $(\hat{\ell}_t(h) - \hat{\ell}_t(h'))$ for each $t \in \calI_l$,
\begin{align*}
    \E_{y \sim \Bernouli(\eta_t^{x_t})} \E_{Q_t} \left[ \hat{\ell}_t(h) - \hat{\ell}_t(h')\right]
     &= \E_{y \sim \Bernouli(\eta_t^{x_t})} \left[\one\{h(x_t) \neq y\} - \one\{h'(x_t) \neq y\}\right] \\
     &\leq \one\{ h(x_t) \neq h'(x_t)\}
\end{align*}
and, 
\begin{align*}
    \Var_t\left(\hat{\ell}_t(h) - \hat{\ell}_t(h') \right)
    &\leq \E_{y \sim \Bernouli(\eta_t^{x_t})} \E_{Q_t} \left[ \left(\hat{\ell}_t(h) - \hat{\ell}_t(h') \right)^2\right] \\
    &=\E_{y \sim \Bernouli(\eta_t^{x_t})} \E_{Q_t} \left[ \frac{\one\{h(x_t) \neq h'(x_t)\}}{(q_l^{x_t})^2}\right]\\
    &= \frac{\one\{h(x_t) \neq h'(x_t)\}}{q_l^{x_t}}\\
    &\leq \frac{\one\{h(x_t) \neq h'(x_t)\}}{\min_{x' \in \Dis(h,h') }q_l^{x'}}
\end{align*}

Then according to the Lemma~\ref{lem: concentratin bound for catoni -- original}, we have
\begin{align*}
    &|(\hat{R}_l(h) -\hat{R}_l(h')) - W_l^{h,h'}|\\
    &\leq \frac{\alpha_l^{h,h'}\left(\frac{\sum_t \one\{ h(x_t) \neq h'(x_t)\}}{\min_{x' \in \Dis(h,h') }q_l^{x'}} + \sum_t \one\{ h(x_t) \neq h'(x_t)\} \right)}{N_l} + \frac{2\log(1/\delta)}{\alpha_l^{h,h'} N_l}\\
    &\leq \frac{2\alpha_l^{h,h'}\hat{\rho}_l(h,h')}{\min_{x' \in \Dis(h,h') }q_l^{x'}}  + \frac{2\log(1/\delta)}{\alpha_l^{h,h'} N_l} \\
    &= \sqrt{ \frac{10\log(1/\delta)\hat{\rho}_l(h,h')}{N_l \min_{x \in \Dis(h,h')} q_l^x}}
\end{align*}
The last one comes from choosing $\alpha_l^{h,h'} = \sqrt{\frac{2\log(1/\delta)\min_{x \in \Dis(h,h')} q_l^x}{5N_l\hat{\rho}_l(h,h')}}$ and also it is easy to verify that 
\begin{align*}
    &(\alpha_l^{h,h'})^2\left(\frac{N_l\hat{\rho}_l(h,h')}{\min_{x' \in \Dis(h,h') }q_l^{x'}} + \sum_t ((R_*(h)-R_*(h'))-(R_t(h)-R_t(h')))^2 \right) + 2\log(1/\delta)\\
    &\leq 4\log(1/\delta) \leq N_l.
\end{align*}
\end{proof}

\subsection{High probability events}
Define the event $\calE_{gap}$ as 
\begin{align*}
    \calE_{gap}
    :=\left\{\forall l, \forall h,h'\in \calH, |(\hat{R}_l(h)-\hat{R}_l(h')) - W_l^{h,h'}|
    \leq \sqrt{ \frac{10\beta_3\hat{\rho}_l(h,h')}{N_l \min_{x \in \Dis(h,h')} q_l^x}} \right\},
\end{align*}
and event $\calE_{dis1},\calE_{dis2}$ as 
\begin{align*}
    &\calE_{dis1}: =
    \left\{ \forall l, \forall h,h'\in \calH, |\hat{\rho}_l(h,h') - \rho_*(h,h')|
    \leq \sqrt{\frac{\beta_3\hat{\rho}_l(h,h')}{N_l}} + \frac{\beta_3}{N_l} \right\}\\
    &\calE_{dis2}:= \left\{ \forall l, \forall h,h'\in \calH,|\hat{\rho}_l(h,h') - \rho_*(h,h')|
    \leq \sqrt{\frac{\beta_3\rho_*(h,h')}{N_l}} + \frac{\beta_3}{N_l}\right\}.
\end{align*}
By condition~\ref{condition: robust estimator} of $\delta$-robust estimator in Algo~\ref{alg:main}, the (empirical) Bernstein inequality and the union bounds, we have easily get $\fP(\calE_{gap} \cap \calE_{dis1}\cap \calE_{dis2}) \geq 1-\delta$ as shown in the following lemmas.

\begin{lemma}
    $\fP(\calE_{est}) \geq 1-\delta/3$
\end{lemma}
\begin{proof}
    We prove this by condition~\ref{condition: robust estimator} in Algo~\ref{alg:main} and the union bound over $|\calH|^2$ number of hypothesis pairs and $\frac{1}{2}\lfloor\log(n) \rfloor$ number of epochs.
\end{proof}

\begin{lemma}
    $\fP(\calE_{gap1}) \geq 1-\delta/3,\fP(\calE_{gap2}) \geq 1-\delta/3$
\end{lemma}
\begin{proof}
    We prove this by (empirical) Bernstein inequality in Algo~\ref{alg:main} and the union bound over $|\calH|^2$ number of hypothesis pairs and $\frac{1}{2}\lfloor\log(n) \rfloor$ number of epochs.
\end{proof}

\subsection{Gap estimation accuracy}
In this section, we show that $\hDelta_h^l$ is close to $\Delta_h$ for all $l,h$. To prove this, we first show some auxiliary lemmas as follows.

\begin{lemma}[Estimation accuracy for $\hat{\calD}_l$]
\label{lem: est helper}
On event $\calE_{gap}$, for any fixed epoch $l$, for any fixed pair $h,h' \in \calH$, suppose $j = \max\{ i | h,h' \in V_l^i\}$, we have
\begin{align*}
    &|(R_{\hat{\calD}_l}(h) - R_{\hat{\calD}_l}(h')) - (R_*(h) - R_*(h'))|\\
    &\leq \frac{1}{16}\left(\max\{ \hDelta_h^{l-1},\hDelta_{h'}^{l-1} \} + \epsilon_{l} \right)
        +\frac{4C_l}{N_l}R^* + \frac{2C_l}{N_l}\max\{ \Delta_h,\Delta_{h'}\}
\end{align*}
\end{lemma}
\begin{proof}

Firstly we show that, for any pair $h,h' \in \calH$ we have
\begin{align*}
    &|(R_{\hat{\calD}_l}(h) - R_{\hat{\calD}_l}(h')) - (R_*(h) - R_*(h'))|\\
    &\leq |(R_{\hat{\calD}_l}(h) - R_{\hat{\calD}_l}(h')) - W_l^{h,h'}| + |W_l^{h,h'} - (\hat{R}_l(h)-\hat{R}_l(h'))| 
        + |(\hat{R}_l(h)-\hat{R}_l(h')) - (R_*(h) - R_*(h'))| \\
    &\leq \max_{h_1,h_2 \in \calH}|\left((R_{\hat{\calD}_l}(h_1) - R_{\hat{\calD}_l}(h_2)) - W_l^{h_1,h_2}| \sqrt{\frac{\min_{x\in\Dis(h_1,h_2)}q_l^x}{\hat{\rho}_l(h_1,h_2)}}\right) \sqrt{\frac{\hat{\rho}_l(h,h')}{\min_{x\in\Dis(h,h')}q_l^x}}\\
        &\quad + |W_l^v - (\hat{R}_l(h)-\hat{R}_l(h'))| 
        + |(\hat{R}_l(h)-\hat{R}_l(h')) - (R_*(h) - R_*(h'))|\\
    &\leq \max_{h_1,h_2 \in \calH}|\left((\hat{R}_l(h_1)-\hat{R}_l(h_2)) - W_l^{h_1,h_2}| \sqrt{\frac{\min_{x\in\Dis(h_1,h_2)}q_l^x}{\hat{\rho}_l(h_1,h_2)}}\right) \sqrt{\frac{\hat{\rho}_l(h,h')}{\min_{x\in\Dis(h,h')}q_l^x}}\\
        &\quad + |W_l^v - (\hat{R}_l(h)-\hat{R}_l(h'))| 
        + |(\hat{R}_l(h)-\hat{R}_l(h')) - (R_*(h) - R_*(h'))|\\
    &\leq 2\max_{h_1,h_2 \in \calH}|\left((\hat{R}_l(h_1)-\hat{R}_l(h_2)) - W_l^{h_1,h_2}| \sqrt{\frac{\min_{x\in\Dis(h_1,h_2)}q_l^x}{\hat{\rho}_l(h_1,h_2)}}\right) \sqrt{\frac{\hat{\rho}_l(h,h')}{\min_{x\in\Dis(h,h')}q_l^x}}\\
        &\quad + |(\hat{R}_l(h)-\hat{R}_l(h')) - (R_*(h) - R_*(h'))|\\
    &\leq  2\sqrt{ \frac{10\beta_3}{N_l}} \sqrt{\frac{\hat{\rho}_l(h,h')}{\min_{x\in\Dis(h,h')}q_l^x}} + |(\hat{R}_l(h)-\hat{R}_l(h')) - (R_*(h) - R_*(h'))|
\end{align*}
The third inequality comes from the definition of $\hat{\calD}_l$ and the last inequality comes from the Condition~\ref{condition: robust estimator} of $\delta$-robust estimator in Algo.~\ref{alg:main}.

For the first term, for any $x \in \Dis(h,h')$, by the definition of $q_l^x$ in line~\ref{line: end calcualting $q$} and the fact that $(h,h') \in \calZ(x)$, we have that,
\begin{align*}
    q_l^x \geq \frac{\beta_1\hat{\rho}_l(h,h')}{N_l}\epsilon_{j}^{-2}. \quad
    \text{, where $j = \max\{i \in [l-1]\mid h,h' \in V_l^i\}$}
\end{align*}
So we can further lower bound the $\min_{x\in\Dis(h,h')}q_l^x$ by 
\begin{align*}
    \min_{x\in\Dis(h,h')}q_l^x 
    \geq \frac{\beta_1\hat{\rho}_l(h,h')}{N_l}\epsilon_{j}^{-2}
    \quad
    \text{, where $j = \max\{i \in [l-1]\mid h,h' \in V_l^i\}$}
\end{align*}
and therefore upper bound the first term as
\begin{align*}
    2\sqrt{  \frac{10\beta_3}{N_l}} \sqrt{\frac{\hat{\rho}_l(h,h')}{\min_{x\in\Dis(h,h')}q_l^x}}
    \leq 2\sqrt{\frac{10\beta_3}{\beta_1}}\epsilon_j.
\end{align*}

For the first term, by the definition of $q_l^x$ in line~\ref{line: end calcualting $q$} and the fact that $(h,h') \in \calZ(x)$, we have that, for any fixed $x$, 
\begin{align*}
    q_l^x \geq \frac{\beta_1\hat{\rho}_l(h,h')}{N_l}\epsilon_{j}^{-2}. \quad
    \text{, where $j = \max\{i \in [l-1]\mid h,h' \in V_l^i\}$}
\end{align*}

For the second term, by the definition of corruptions, we have
\begin{align*}
    &|(\hat{R}_l(h)-\hat{R}_l(h')) - (R_*(h) - R_*(h'))|\\
    &\leq|(\hat{R}_l(h)-\hat{R}_l(h')) - (\avgR_l(h)-\avgR_l(h'))| + |(\avgR_l(h)-\avgR_l(h')) - (R_*(h) - R_*(h'))|\\
    &\leq 2\sqrt{\frac{\beta_3}{N_l}}+\frac{2C_l}{N_l}\rho_*(h,h')\\
    &\leq 2\sqrt{\frac{\beta_3}{\beta_1}}\epsilon_l+\frac{2C_l}{N_l}\left(\rho_*(h,h^*) + \rho_*(h',h^*)\right)\\
    &\leq 2\sqrt{\frac{\beta_3}{\beta_1}}\epsilon_l+\frac{4C_l}{N_l}R^* + \frac{2C_l}{N_l}\max\{ \Delta_h,\Delta_{h'}\}
\end{align*}
where the second inequality comes from Bernstein inequality and Lemma~\ref{lem: corruption related  2}.

Finally we are going to make the connection between $\epsilon_j$ and the $\hDelta_h^{l-1},\hDelta_{h'}^{l-1}$. Note that if $j < {l-1}$, by definition of $j$, we must have
$h,h' \not in  V_l^{j+1}$. By the definition that $\forall h \notin V_{l+1}^i, \hDelta_h^l \geq \epsilon_i$, we have
\begin{align*}
    \max\{ \hDelta_h^{l-1},\hDelta_{h'}^{l-1} \} > \epsilon_{j+1} = \frac{\epsilon_j}{2}.
\end{align*}
and if $j = {l-1}$, we directly have $\frac{\epsilon_j}{2} \leq \epsilon_l$.
Therefore, we have $\frac{\epsilon_j}{2} \leq \max\{ \hDelta_h^{l-1},\hDelta_{z'}^{l-1} \} + \epsilon_{l}$.
\end{proof}

\begin{lemma}[Upper bound of the estimated gap]
\label{lem: upper bound of the estimated gap}
On event $\calE_{gap}$, for any fixed epoch $l$, suppose its previous epoch satisfies that, for all $h \in \calH$,
\begin{align}
    &\Delta_h \leq \frac{3}{2}\hDelta_h^{l-1} + \frac{3}{2}\epsilon_{l-1} + 3g_{l-1},\\
    &\hDelta_h^{l-1} \leq 2\left( \Delta_h + \epsilon_{l-1} + g_{l-1} \right),
\end{align}
then we have,
\begin{align*}
    \hDelta_h^l \leq 2\left( \Delta_h + \epsilon_l + g_l \right)
\end{align*}
where 

. 
\end{lemma}
\begin{proof}
According to the definition of $\hDelta_h^l$, If $\inprod{h-\hat{h}_*^l,\hat{\theta}_l} - \beta_2 \hat{\Delta}_{\hat{h}_*^l}^{l-1} \leq \epsilon_l$, then the above trivially holds, Otherwise, we have
\begin{align*}
    \hDelta_h^l
    & = R_{\hat{\calD}_l}(h) -\left( R_{\hat{\calD}_l}(\hat{h}_*^l) + \beta_2 \hat{\Delta}_{\hat{h}_*^l}^{l-1} \right)\\
    & = \left((R_{\hat{\calD}_l}(h) - R_{\hat{\calD}_l}(\hat{h}_*^l)) - (R_*(h)-R_*(\hat{h}_*^l))\right) 
        + (R_*(h)-R_*(\hat{h}_*^l)) -\beta_2\hat{\Delta}_{\hat{h}_*^l}^{l-1} \\
    &\leq \left((R_{\hat{\calD}_l}(h) - R_{\hat{\calD}_l}(\hat{h}_*^l)) - (R_*(h)-R_*(\hat{h}_*^l))\right) 
        + \Delta_h -\beta_2\hat{\Delta}_{\hat{h}_*^l}^{l-1} \\
    & \leq \frac{1}{16}\left(\max\{ \hDelta_h^{l-1},\hDelta_{\hat{h}_*^l}^{l-1} \} + \epsilon_{l} \right)
        + \frac{1}{16} \max\{\Delta_h,\Delta_{\hat{h}_*^l}\}
        + \Delta_h - \beta_2 \hat{\Delta}_{\hat{h}_*^l}^{l-1} \\
       &\quad + \underbrace{\frac{4C_l}{N_l}R^*\one\{\frac{2C_l}{N_l} \leq \frac{1}{16}\} + \frac{2C_l}{N_l}\one\{\frac{2C_l}{N_l} > \frac{1}{16}\}}_\text{Corruption Term}\\
    & = \frac{1}{16}(\hDelta_h^{l-1} + \epsilon_{l})
        +\frac{1}{16}\Delta_h
       +\frac{1}{16}\hDelta_{\hat{h}_*^l}^{l-1}
        + \frac{1}{16}\Delta_{\hat{h}_*^l}
       -\beta_2 \hat{\Delta}_{\hat{h}_*^l}^{l-1}
       + \Delta_h + \text{Corruption Term }\\ 
    &\leq \left(\frac{1}{16}(\hDelta_h^{l-1} + \epsilon_{l})
        +\frac{1}{16}\Delta_h
        + \Delta_h \right)
         + \left(\frac{1}{16}\hDelta_{\hat{h}_*^l}
      +\frac{3}{32}\hDelta_{\hat{h}_*^l}^{l-1}
      -\beta_2 \hat{\Delta}_{\hat{h}_*^l}^{l-1}\right)
      +\frac{3}{32}(\epsilon_{l-1} + 2g_{l-1}) + \text{Corruption Term }
        \\
    &\leq \left(\frac{1}{16}(\hDelta_h^{l-1} + \epsilon_{l})
        +\frac{1}{16}\Delta_h+ \Delta_h \right)
        +\frac{3}{32}(\epsilon_{l-1} + 2g_{l-1})
        + \text{Corruption Term }
        \label{line: use beta assump}\\
    & = \frac{1}{16}\hDelta_h^{l-1}
        + \left(1+ \frac{1}{16}\right)\Delta_h 
        +\frac{1}{4}\epsilon_l + 4R^*\frac{C_l}{N_l}
             +\frac{3}{16}g_{l-1}\\
    & \leq 2(\Delta_h + \epsilon_l + g_l)
\end{align*}
Here the first inequality comes from the definition of $h^*$, the second inequality comes from Lemma~\ref{lem: est helper}, the third inequality comes from the the assumption (1) and the penultimate inequality  comes from the fact that $\beta_2 \geq \frac{5}{32}$. Finally, the last inequality comes from assumption (2).

\end{proof}

\begin{lemma}[Lower bound of the estimated gap]
\label{lem: lower bound of the estimated gap}
On event $\calE_{gap}$, for any fixed epoch $l$, suppose the following holds, for all $h \in \calH$,
\begin{align}
    \hDelta_h^{l-1} \leq 2\left( \Delta_h + \epsilon_{l-1} + g_{l-1} \right),
\end{align}
then we have,
\begin{align*}
    \Delta_h \leq \frac{3}{2}\hDelta_h^l + \frac{3}{2}\epsilon_l + 3g_l
\end{align*}
\end{lemma}
\begin{proof}
\begin{align*}
    \hDelta_h^l
    & \geq R_{\hat{\calD}_l}(h) -\left( R_{\hat{\calD}_l}(h^*) + \beta_2 \hat{\Delta}_{h^*} \right)\\
    & = \left((R_{\hat{\calD}_l}(h) - R_{\hat{\calD}_l}(h^*)) - (R_*(h)-R^*)\right) 
        + \Delta_h -\beta_2\hat{\Delta}_{h^*}^{l-1} \\
    & \geq -\frac{1}{16}(\hDelta_h^{l-1} + \epsilon_{l})
        -\frac{1}{16}\Delta_h 
        - \frac{1}{16}\hDelta_{h^*}^{l-1}
      -\frac{1}{16}\Delta_{h^*}
       -\beta_2 \hat{\Delta}_{h^*}^{l-1}
       + \Delta_h\\
        &\quad - \underbrace{\left(4R^*\frac{C_l}{N_l}\one\left\{\frac{2C_l}{N_l}\leq \frac{1}{16}\right\} + \frac{C_l}{N_l}\one\left\{\frac{2C_l}{N_l}> \frac{1}{16}\right\}\right)}_\text{Corruption Term}\\
    &= -\frac{1}{16}(\hDelta_h^{l-1} + \epsilon_{l})
        -\frac{1}{16}\Delta_h
        - \frac{1}{16}\hDelta_{h^*}
       -\beta_2 \hat{\Delta}_{h^*}^{l-1}
       + \Delta_h- \text{Corruption Term}\\
    &\geq-\frac{1}{16}(2\Delta_h + 2\epsilon_{l-1}+2g_{l-1} + \epsilon_{l})
        +\Delta_h
        - (\frac{1}{16}+\beta_2)(2\epsilon_{l-1}+2g_{l-1})
        - \text{Corruption Term}\\
    & \geq \frac{13}{16}\Delta_h - \frac{38}{32}\epsilon_l - \frac{18}{32}g_{l-1} 
    - 4R^*\frac{C_l}{N_l} 
    - \text{Corruption Term}
    \\
    & \geq \frac{13}{16}\Delta_h - \frac{38}{32}\epsilon_l - \frac{18}{8}g_l
\end{align*}
Here the first inequality comes from the definition of $\hat{h}_*^l$, the second inequality comes from Lemma~\ref{lem: est helper}. and the third inequality comes from the upper bound of the estimated gap in Lemma~\ref{lem: upper bound of the estimated gap}.
\end{proof}

\textbf{Now we are ready to prove the final key lemma, which shows that such upper bound and lower bound for $\hDelta_h^l$ holds for all $l$ and $h$.}

\begin{lemma}[Upper bound and lower bound for all estimation]
\label{lem: upper and lower bound}
  On event $\calE_{gap}$, for any epoch $l$, for all $h \in \calH$,
  \begin{align}
      &\hDelta_h^l \leq 2\left( \Delta_h + \epsilon_l + g_l \right) \\
      &\Delta_h \leq \frac{3}{2}\hDelta_h^l + \frac{3}{2}\epsilon_l + 3g_l
  \end{align}
\end{lemma}
\begin{proof}
We prove this by induction. 

For the base case where $l=1$. we can easily have the following
\begin{align*}
    \hDelta_h^{1}
    \leq 1
    \leq 2\Delta_h + 2\epsilon_{1} + 2g_l
\end{align*}
and also, by using Lemma~\ref{lem: lower bound of the estimated gap} and the fact that 
$\hDelta_h^{0} \leq 2(\Delta_h + \epsilon_{0} + g_{0})$, it is easy to get 
\begin{align*}
    \Delta_h \leq \frac{3}{2}\hDelta_h^{1} + \frac{3}{2}\epsilon_{1} + 3g_{1}
\end{align*}
So the target inequality holds for $l=1$.

Suppose the target inequality holds for $l'-1$ where $l' \geq 2$, then by Lemma~\ref{lem: upper bound of the estimated gap}, we show that the first target inequality holds for $l'$. Also by Lemma~\ref{lem: lower bound of the estimated gap}, we show that the second target inequality holds for $l'$. Therefore, we finish the proof. 

\end{proof}

\subsection{Auxiliary lemmas}

\begin{lemma}
\label{lem: upper bound the difference of classifiers}
For any epoch $l$ and layer $j$, we have
\begin{align*}
    \max_{h \in V_l^j} \rho_*(h,h^*)
    \leq 2R^* + 3\epsilon_j + 3g_{l-1}
\end{align*}
\end{lemma}
\begin{proof}
\begin{align*}
    \max_{h \in V_l^j} \rho_*(h,h^*)
    &\leq 2R^* + \max_{h \in V_l^j} \Delta_h\\
    &\leq 2R^* 
        + \max_{h \in V_l^j} \left( \frac{3}{2}\hDelta_h^{l-1} + \frac{3}{2}\epsilon_{l-1} + 3g_{l-1}\right) \\
    &\leq 2R^* + 3\epsilon_j + 3g_{l-1}
\end{align*}
The first inequality comes from the fact the
$\rho_*(h,h^*) \leq R_*(h) + R^* = 2R^* + \Delta_h$, the second inequality comes form the lower bound in Lemma~\ref{lem: upper and lower bound} and the last inequality is by the definition of $ V_l^j$.
\end{proof}


\subsection{Main proof for Theorem~\ref{them: main}}
\label{sec(app): main proof for main algo}
Here we assume $\log_4(\frac{n}{\beta_1}) \notin \fN$ and there are no corruptions in the last unfinished epoch $\lceil \log_4(\frac{n}{\beta_1}) \rceil$. This will not effect the result but will make the proof easier. Given that events $\calE_{gap},\calE_{dis1}$  and $\calE_{dis2}$, then we have the following proofs.

\textbf{ First we deal with the sample complexity.} 

For any $t \in \calI_l$,the probability of $x_t$ being queried ($Q_t$) is
\begin{align*}
    \E[Q_t]
    & = \sum_{x \in \calX} P(x_t = x) q_l^x\\
    & = \sum_{x \in \calX} P(x_t = x) \max_{(h,h') \in \calZ(x)}\frac{\beta_1\hat{\rho}_{l-1}(h,h')}{N_l}\epsilon_{k(h,h',l)}^{-2}\\
    & \leq \frac{\beta_1}{N_l}\sum_{x \in \calX}  P(x_t = x) \max_{(h,h') \in \calZ(x)}\rho_*(h,h')\epsilon_{k(h,h',l)}^{-2} \\
        &\quad + 4\frac{\beta_1}{N_l}\sum_{x \in \calX}  P(x_t = x)\sqrt{\rho_*(h,h')\epsilon_{k(h,h',l)}^{-2}} 
        + \frac{4\beta_1}{N_l}\\
    & \leq 5\frac{\beta_1}{N_l}\sum_{x \in \calX}  P(x_t = x) \max_{(h,h') \in \calZ(x)}\rho_*(h,h')\epsilon_{k(h,h',l)}^{-2}
        + 8\frac{\beta_1}{N_l}\\
    & = 5\frac{\beta_1}{N_l}\sum_{x \in \calX}  P(x_t = x) \rho_*(h_1^x,h_2^x)\epsilon_{j^x}^{-2} + 8\frac{\beta_1}{N_l}\\
    & \leq 5\frac{\beta_1}{N_l}\sum_{x \in \calX}  P(x_t = x) \max_{h_3,h_4 \in V_l^{j^x}} \rho_*(h_3,h_4)\epsilon_{j^x}^{-2} + 8\frac{\beta_1}{N_l}\\
    & \leq 10\frac{\beta_1}{N_l}\sum_{x \in \calX}  P(x_t = x) \max_{h \in V_l^{j^x}} \rho_*(h,h^*)\epsilon_{j^x}^{-2} + 8\frac{\beta_1}{N_l} \\
    & \leq 10\frac{\beta_1}{N_l}\sum_{x \in \calX}  P(x_t = x) \left( 2R^*\epsilon_{j^x}^{-2} + 3\epsilon_{j^x}^{-1} + 3g_{l-1}\epsilon_{j^x}^{-2} \right) + 8\frac{\beta_1}{N_l} \\
    & = 10\frac{\beta_1}{N_l} \sum_{i=1}^{l-1} \left( 2R^*\epsilon_{i}^{-2} + 3\epsilon_{i}^{-1} + 3g_{l-1}\epsilon_{i}^{-2} \right)\sum_{x \in \calX}  P(x_t = x) \one\{j^x=i\}  + 8\frac{\beta_1}{N_l} \\
    &\leq 10\frac{\beta_1}{N_l} \sum_{i=0}^{l-1} \left( 2R^*\epsilon_{i}^{-2} + 3\epsilon_{i}^{-1} + 3g_{l-1}\epsilon_{i}^{-2} \right) \fP( x \in   \Dis(V_l^i) )+ 8\frac{\beta_1}{N_l}
\end{align*}
Here $(h_1^x,h_2^x)= \argmax_{(h,h') \in \calZ(x)}\rho_*(h,h')\epsilon_{k(h,h',l)}^{-2}$ and $j^x = k(h_1^x,h_2^x,l)$. The first inequality comes from the event $\calE_{dis2}$, the second inequality comes from the fact that 
$\sqrt{\rho_*(h,h')\epsilon_{k(h,h',l)}^{-2}} \leq \rho_*(h,h')\epsilon_{k(h,h',l)}^{-2} + 1$
and penultimate inequality comes from the Lemma~\ref{lem: upper bound the difference of classifiers}.


Now we can use the standard techniques to bound $\fP( x \in   \Dis(V_l^i) )$ as follows
\begin{align*}
    \fP( x \in   \Dis(V_l^i) )
    & = \fP\left(\exists h,h' \in V_l^i: h(x)\neq h'(x) \right)\\
    & \leq \fP\left(\exists h \in V_l^i: h(x)\neq h^*(x) \right) \\
    &\leq \fP\left(\exists h \in \calH: h(x)\neq h^*(x), \rho_*(h,h^*) \leq 2R^* + 3\epsilon_i + 3g_{l-1} \right)\\
    &\leq \theta^*(2R^* + 3\epsilon_i + g_{l-1}) \left( 2R^* + 3\epsilon_i + 3g_{l-1}\right)
\end{align*}
where again the first inequality comes from Lemma~\ref{lem: upper bound the difference of classifiers}.

Combine with the above result, we get the expected number of queries inside a complete epoch $l$ as,
\begin{align*}
    \sum_{t \in \calI_l} \E[Q_t]
    &= 10\beta_1 \sum_{i=0}^{l-1} \theta^*(2R^* + 3\epsilon_i + g_{l-1})\\
        &\quad  * \left(4(R^*)^2\epsilon_i^{-2} + 12 R^*\epsilon_i^{-1} + 12 R^*g_{l-1}\epsilon_i^{-2} + 18g_{l-1}\epsilon_i^{-1} + 9g_{l-1}^2\epsilon_i^{-2} + 9\right) \\
    &\leq 20\beta_1 \theta^*(2R^* + 3\epsilon_{l-1} + g_{l-1})\\
        &\quad  *\left(4(R^*)^2\epsilon_l^{-2} + 12 R^*\epsilon_l^{-1} + \frac{24}{\beta_1} R^*\bar{C}_{l-1} + \frac{36}{\beta_1}\bar{C}_{l-1}\epsilon_{l-1} + \frac{36}{\beta_1^2}\bar{C}_{l-1}^2\epsilon_{l-1}^2 + 9\right)\\
    &\leq 20\beta_1 \theta^*(2R^* + 3\epsilon_{l-1} + g_{l-1})
      *\left(4(R^*)^2\epsilon_l^{-2} + 12 R^*\epsilon_l^{-1} + \frac{132}{\beta_1}\bar{C}_{l-1} + 10\right)
\end{align*}
where the second inequality comes from the fact that 
$g_l = \frac{2}{\beta_1}\epsilon_l^2\bar{C}_{l} $  and the third inequality comes from that fact that $\bar{C}_{l-1} \leq \sum_{s=1}^{l-1} C_s \leq 2\beta_1 \epsilon_{l-1}^{-2}$.

Summing over all $L = \lceil \frac{1}{2}\log(n/\beta_1) \rceil$ number of epochs, we have that, for any $n$,
\begin{align*}
    &\text{Query complexity}\\
    &\leq \sum_{l=1}^{L}\sum_{t \in \calI_l} \E[Q_t] \\
    &\leq 40\beta_1 \theta^*(2R^* + 3\epsilon_{L-1}+ g_{L-1})\left(4(R^*)^2\epsilon_{L}^{-2} + 12 R^*\epsilon_{L}^{-1}\right)\\
        &\quad + 40\beta_1 \theta^*(2R^* + 3\epsilon_{L-1}+ g_{L-1})L\left(\frac{132}{\beta_1}\bar{C}_{total} + 10\right)\\
    &=40\beta_1 \theta^*(2R^* + 3\epsilon_{L-1}+ g_{L-1})
        \left(4(R^*)^2\frac{n}{\beta_1} + 12 R^*\sqrt{\frac{n}{\beta_1}} + 5\log(n/\beta_1)\right)\\
        &\quad +2450 \theta^*(2R^* + 3\epsilon_{L-1}+ g_{L-1})\log(n/\beta_1)\bar{C}_{total}\\
    &=\theta^*(2R^* + 3\epsilon_{L-1}+ g_{L-1})  \left(160(R^*)^2n + 480 R^*\sqrt{n \beta_1} + 200\beta_1\log(n/\beta_1)\right)\\
        &\quad+2450 \theta^*(2R^* + 3\epsilon_{L-1}+ g_{L-1})\log(n/\beta_1)\bar{C}_{total}\\
    &\leq \order\left(\theta^*(R^* + 3\sqrt{\frac{\beta_1}{n}} + \frac{\bCtotal}{n})\left((R^*)^2n + \log(n/\beta_1)\right)\beta_1 \right) \\
        &\quad+ \order\left(\theta^*(R^* + 3\sqrt{\frac{\beta_1}{n}} + \frac{\bCtotal}{n})\log(n/\beta_1)\bar{C}_{total}\right) 
\end{align*}
where the last inequality comes from the following lower bound,
\begin{align*}
    3\epsilon_{L-1}+ g_{L-1}
    = 3\epsilon_{L-1}  + \frac{2}{\beta_1}\bCtotal \epsilon_{L-1}^2
    \geq 3\sqrt{\frac{\beta_1}{n}} + \frac{2\bCtotal}{n}
\end{align*}

\textbf{Now we will deal with the correctness. } By Lemma~\ref{lem: upper and lower bound}, we have
\begin{align*}
    \Delta_{h_{out}}
    &\leq \frac{3}{2}\hDelta_{h_{out}}^{L-1} + \frac{3}{2}\epsilon_{L-1} + 3g_{L-1} \\
    &\leq 3\epsilon_{L-1} + 3g_{L-1} \\
    &\leq 6\sqrt{\frac{2\beta_1}{n}}+ 3g_{L-1} \\
    &\leq 6\sqrt{\frac{2\beta_1}{n}} + 24\frac{\bar{C}_{total}}{n}
\end{align*}
where the second inequality comes from the definition of $\hout$ and $V_L^{L-1}$ and the third and last inequality is just by replacing the value of $\epsilon_{L-1}$ and $g_{L-1}$. 
\textbf{Finally, we can written this result in the $\varepsilon$-accuracy form.} Set $6\sqrt{\frac{2\beta_1}{n}} := \varepsilon$, we have $n = \frac{72\beta_1}{\varepsilon^2}$.

\end{document}